\documentclass[11pt]{article} 
\usepackage[lmargin=1.25in,rmargin=1.25in, bmargin=1.25in, tmargin=1.25in]{geometry}

\usepackage[normalem]{ulem} 
\usepackage{algpseudocode}
\usepackage[ruled]{algorithm2e}
\usepackage{mathtools}
\usepackage{hyperref}
\usepackage{bbm}
\usepackage[round]{natbib}
\usepackage{url}
\usepackage{graphicx}
\usepackage{amsmath,amsthm,amssymb,amsfonts}

\newif\ifcomments
\commentstrue 

\newcommand{\N}{\mathbb{N}}
\newcommand{\naturals}{\mathbb{N}}

\newcommand{\st}{\;.\;}
\newcommand{\set}[1]{\left\{#1\right\}}
\newcommand{\given}{\;|\;}

\DeclareMathOperator*{\argmin}{arg\,min}

\newcommand{\eval}[2]{\underset{{#1}}{\mathbb{E}}\left[#2\right]}
\newcommand{\prob}[2]{\underset{{#1}}{\Pr}\left(#2\right)}
\newcommand{\supp}{\text{supp}}

\newcommand{\VC}{\mathsf{VC}}
\newcommand{\Lit}{\mathsf{Lit}}
\newcommand{\R}{\mathsf{R}}

\newcommand{\mar}{\text{mar}}
\newcommand{\erm}{\mathsf{ERM}}
\newcommand{\rerm}{\mathsf{RERM}}
\newcommand{\twohalt}{\mathsf{TwoHalt}}

\newcommand{\zjn}{z_j^{0}}
\newcommand{\zip}{z_i^{1}}
\newcommand{\zin}{z_i^{0}}
\newcommand{\zset}[1][m]{\left\{(\zin,\zip)\right\}_{i=1}^{#1}}
\newcommand{\zdef}[1][m]{Z=\zset[#1]}
\newcommand{\cdim}{c\text{-}\dim_\U}
\newcommand{\cvc}{c\text{-}\VC}
\newcommand{\A}{\mathcal{A}}
\renewcommand{\H}{\mathcal{H}}
\renewcommand{\P}{\mathcal{P}}
\newcommand{\U}{\mathcal{U}}
\newcommand{\X}{\mathcal{X}}
\newcommand{\Y}{\mathcal{Y}}

\newcommand{\indct}[1]{\mathbbm{1}\left[{#1}\right]}

\newtheorem{theorem}{Theorem}
\newtheorem{proposition}[theorem]{Proposition}
\newtheorem{lemma}[theorem]{Lemma}
\newtheorem{corollary}[theorem]{Corollary}
\newtheorem{example}[theorem]{Example}
\newtheorem{definition}[theorem]{Definition}

\newtheorem{remark}[theorem]{Remark}

\newtheorem{fact}[theorem]{Fact}

\title{On the Computability of Robust PAC Learning}

\usepackage{authblk}
\author{Pascale Gourdeau\thanks{Vector Institute \& University of Toronto, Toronto, ON, Canada; \texttt{pascale.gourdeau@vectorinstitute.ai}} ,  Tosca Lechner\thanks{Cheriton School of Computer Science, University of Waterloo, Waterloo, ON, Canada; \texttt{tlechner@uwaterloo.ca}} ,  and Ruth Urner\thanks{Lassonde School of Engineering, EECS Department, York University, Toronto, ON, Canada; \texttt{ruth@eecs.yorku.ca}} \thanks{Alphabetical order.}}
\date{\vspace{-5ex}}
\begin{document}

\maketitle

\begin{abstract}%
We initiate the study of computability requirements for adversarially robust learning. Adversarially robust PAC-type learnability is by now an established field of research.
However, the effects of computability requirements in PAC-type frameworks are only just starting to emerge. 
We introduce the problem of robust computable PAC (robust CPAC) learning and provide some simple sufficient conditions for this. 
We then show that learnability in this setup is not implied by the combination of its components: classes that are both CPAC and robustly PAC learnable are not necessarily robustly CPAC learnable. 
Furthermore, we show that the novel framework exhibits some surprising effects: for robust CPAC learnability it is not required that the robust loss is computably evaluable!
Towards understanding characterizing properties, we introduce a novel dimension, the computable robust shattering dimension. 
We prove that its finiteness is necessary, but not sufficient for robust CPAC learnability. 
This might yield novel insights for the corresponding phenomenon in the context of robust PAC learnability, where insufficiency of the robust shattering dimension for learnability has been conjectured, but so far a resolution has remained elusive.
\end{abstract}


\section{Introduction}

Formal studies of learnability mostly fall into one of two extremes: the focus is either on purely statistical (or information-theoretic) aspects, where predictors and learners are treated as functions; or requirements of computational efficiency, namely runtime that is polynomial in various parameters, are imposed. Recent work has introduced the study of learnability under a more basic, yet arguably essential requirement, namely that both learners and predictors are computable functions \citep{agarwal2020learnability}, a framework termed Computable Probably Approximately Correct (CPAC) learnability. For binary classification, a setting where various versions of standard learnability (realizable, agnostic, proper, improper, learnable by any ERM) are well known to be characterized (in the information-theoretic sense) by finiteness of the VC dimension, the addition of computability requirements has revealed a somewhat more fine-grained landscape:  CPAC learnability has been shown to be equivalent to a computable version of the VC dimension, the so called \emph{effective VC dimension}, while proper CPAC learnability was proven to be equivalent to the finiteness of the VC dimension combined with the existence of a computable \emph{approximate} ERM learner \citep{sterkenburg2022characterizations, delle2023find}.

In this work, we initiate the study of computable PAC learning in the (adversarially) robust setting. 
Adversarial robust PAC learning has by now been extensively studied, e.g., by \citet{montasser2019vc, montasser2021adversarially, gourdeau2021hardness, awasthiMM023, lechner2023adversarially}
and recent work has provided a characterization of learnability (again, in the purely information-theoretic sense) by a parameter of the one inclusion graph of the learning problem \citep{montasser2022adversarially}. We here explore how adding a requirement of computability adds more subtle aspects to the questions of robust PAC learnability, and expose some perhaps surprising aspects of this setup.

We start by setting up a formal framework for robust learning under computability constraints on the learners and hypotheses, and also the perturbation sets employed. As a warm-up, we provide some simple sufficient conditions for learnability in our setting, but also show that the question of robust CPAC learnability is more subtle than the combination of its components: we prove that there exist classes that are CPAC learnable and robustly PAC learnable with respect to perturbation types that are decidable, yet are not robustly CPAC learnable. 

We then explore the role of computability of the robust loss for robust CPAC learnability. 
Perhaps surprisingly, the robust loss being computably evaluable is neither a necessary nor a sufficient requirement for robust learnability.
The insufficiency result relies on showing that there are CPAC learnable classes with computably evaluable robust loss that do not admit a robust empirical risk minimization (RERM) oracle \citep{montasser2019vc}  which would return a hypothesis with minimal empirical robust risk on any sample) in the robust agnostic setting.
This result is significant as many works assume access to an RERM oracle. In particular, the standard reduction of agnostic to realizable learning in the robust learning setting requires such an oracle \citep{montasser2019vc}, and we here demonstrate an example where learnability in the robust realizable case does not extend to the agnostic case through this reduction, since RERM is not computable. 

Finally, we explore the role of dimensions for computable robust learnability. We introduce a novel, computable version of the robust shattering dimension. Finiteness of  the robust shattering dimension has been shown to to be necessary for robust learnability \citep{montasser2019vc}, however it has remained an open question whether it is also a sufficient condition \citep{montasser2022adversarially}. 
For the computable version, we prove that the finiteness of our \emph{computable robust shattering dimension} is necessary, yet not sufficient, to ensure robust CPAC learnability. This result might provide novel insights for the case without computability constraints. Our results involve deriving a new no free lunch theorem for robust learning, which might also be of independent interest.

\subsection{Related Work}

\paragraph{Computable Learning.} 
Incorporating aspects of computability into formal learning frameworks is a very recent field of study, originally motivated by the establishement of problems whose learnability is independent of set-theory \citep{ben2017learning}. The introduction of the notion of computable PAC (CPAC) learning \citep{agarwal2020learnability} provided a framework for standard binary classification. Follow up works have by now resulted in a full characterization of CPAC learnability  in terms of the so-called effective VC dimension \citep{sterkenburg2022characterizations, delle2023find}. Very recent works have extended the study of computability in learning to other learning settings such as continuous domains \citep{ackerman2022computable} and online learning \citep{hasrati2023computable}.

\paragraph{Robust Learning.} There is a rich learning theory literature on adversarial robustness. Earlier work focused on the existence of adversarial examples for the stability notion of robustness \citep{fawzi2016robustness,fawzi2018adversarial,fawzi2018analysis,gilmer2018adversarial,shafahi2018adversarial,tsipras2019robustness}, and for its true-label counterpart \citep{diochnos2018adversarial,mahloujifar2019curse}. Many of the works in the former category highlighted the incompatibility between robustness and accuracy, while \citet{BhattacharjeeC21} and \citet{chowdhury2022robustness} later argued that these two objectives should not be in conflict by proposing notions of adaptive robusteness.
Some studies focus on the sample complexity of robust learning in a PAC setting through a notion called the adversarial VC dimension \citep{cullina2018pac}. Some upper bounds depend on the number of perturbations allowed for each instance \citep{attias2022improved}, others on the VC and dual VC dimension of a hypothesis class obtained through an improper learner \citep{montasser2019vc}, thus also covering the case of infinite perturbation sets. 
\citet{ashtiani2020black} later gave an upper bound for \emph{proper} robust learning in terms of the VC dimensions of the class and induced margin class. 
\citet{montasser2022adversarially} exhibited a characterization of robust learning based on the one-inclusion graph of \citet{haussler1994predicting} adapted to robust learning.
The sample complexity of robust learning has also been studied in the semi-supervised setting \citep{attias2022characterization}, through the lens of transformation invariances \citep{shao2022theory}, through Rademacher complexity bounds \citep{khim2019adversarial,yin2019rademacher,awasthi2020adversarial} and by the use of online learning algorithms \citep{diakonikolas2020complexity,bhattacharjee2021sample}.
Robust risk relaxations have been studied by  \citet{viallard2021pac,ashtiani2023adversarially,bhattacharjee2023robust,raman2023proper}, while \citet{balcan2022robustly,balcan2024reliable} offered reliability guarantees for robust learning under various notions of robustness.
In terms of the true-label robust risk, sample complexity upper and lower bounds with access to random examples only have been derived with respect to distributional assumptions \citep{diochnos2020lower,gourdeau2021hardness,gourdeau2022sample}, as this set-up does not allow distribution-free robust learning \citep{gourdeau2021hardness}.
To circumvent computational or information-theoretic obstacles of robust learning with random examples only, recent work has studied robust learning with the help of oracles  \citep{montasser2021adversarially,gourdeau2022local,lechner2023adversarially}, while other work has instead looked at curtailing the (computational) power of the adversary \citep{mahloujifar2019can,garg2020adversarially}.  
Prior works have also analyzed the role of computational efficiency in robust learnability \citep{bubeck2019adversarial, degwekar2019computational, gourdeau2021hardness}, however we are the first to exhibit effects of issues arising from the basic, essential
requirement of computability.

\section{Problem Set-up}
\label{sec:setup}
\paragraph{Notation.} Denote by $\Sigma$ a finite alphabet, with  $\Sigma^*$ denoting the set of all finite words, or strings, over $\Sigma$. We use standard notation for sets and functions, for example, we let $[n] = \{1,2,\ldots, n\}\subset \N$ denote the set of the first $n$ natural numbers.

\paragraph{Computability.} 
Throughout this paper, we will assume we fixed a programming language. 
The existence of a program or algorithm is then equivalent to the existence of a Turing machine.
We use $(T_i)_{i\in \N}$ to denote a fixed enumeration of all Turing machines/programs. We use the notation $T \equiv S$ to indicate that two Turing Machines have the same behavior as functions (not necessarily the same encoding, thus not necessarily the same index in the enumeration). Further, for a Turing machine $T$ and input $x$, we write $T(x)\downarrow$ to indicate that $T$ halts on input $x$ and we write $T(x)\uparrow$ to indicate that $T$ loops (does not halt) on input $x$.

We say that a function $f:\Sigma^*\rightarrow\Sigma^*$ is \emph{total computable} if there exists a program $P$ that halts on every string $\sigma\in\Sigma^*$ with $P(\sigma)=f(\sigma)$.
A set $S\subseteq \Sigma^*$ is called \emph{decidable} (or recursive) if there exists a program $P$ such that for every $\sigma\in\Sigma^*$, $P$ halts on $\sigma$ and outputs whether $\sigma\in S$. 
Finally, a set $S\subseteq \Sigma^*$ is called \emph{recursively enumerable} (or semidecidable) if there exists a program $P$ that enumerates all strings in $S$, or, equivalently, if there exists a program $P$ that halts on every input $\sigma\in S$ and, if it halts, correctly indicates whether $\sigma\in S$.

Moreover, we will fix a language and proof system for first-order logic that is both sound and complete: a first-order formula has a proof if and only if it is a tautology.
We additionally require that this language has a rich enough vocabulary (with respect to function and relation symbols) to guarantee that the set of all its tautologies is undecidable (e.g. \citet{MendelsonIntroToLogic}, Proposition 3.54 (Church's Theorem)). 

\paragraph{Learnability.} Let $\X$ be the input space and $\Y$ the label space. 
We will focus on the case $\Y=\{0,1\}$, i.e., binary classification, and countable domains, thus $\X=\N$.
We let $\H\subseteq \{0,1\}^\N$ denote a \emph{hypothesis class} on $\X$.
Given a joint distribution $D$ on $\X\times\Y$ and a hypothesis $h\in\H$, the \emph{risk} (or error) of $h$ with respect to $D$ is defined as 
$$\R(h;D)=\prob{(x,y)\sim D}{h(x)\neq y}\enspace.$$
We will also denote by $\ell:\H\times\X\times\{0,1\}\rightarrow\{0,1\}$ the 0-1 loss function $\ell(h,x,y)=\mathbf{1}[h(x)\neq y]$,  
and we use the notation $\R(h;S)= \frac{1}{m}\sum_{i=1}^m\ell(h,x_i,y_i)$ to denote the \emph{empirical risk} of $h$ on a sample $S = \{(x_i, y_i)\}_{i=1}^m \in (\X\times \Y)^m$.
We operate in the PAC-learning framework of \citep{valiant1984theory}.

\begin{definition}[Agnostic PAC Learnability] 
    A hypothesis class $\H$ is \emph{PAC learnable in the agnostic setting} if there exists a learner $\A$ and function $m(\cdot,\cdot)$ such that for all $\epsilon,\delta\in(0,1)$ and  for any distribution $D$, if the input to $\A$ is an i.i.d. sample $S$ from $D$ of size at least $m(\epsilon,\delta)$, then, with probability at least $(1-\delta)$ over the samples, the learner outputs a hypothesis $\A(S)$ with 
    $\R(\A(S);D)\leq \underset{h\in\H}{\inf}\R(h;D)+\epsilon\enspace.$
    The class is said to be \emph{PAC learnable in the realizable setting} if the above holds under the condition that $\underset{h\in\H}{\inf}\R(h;D)=0$.
\end{definition}

It is well known that a (binary) hypothesis class is PAC learnable if and only if it has finite \emph{VC dimension} \citep{vapnik1971uniform}, which is defined below. 

\begin{definition}[Shattering and VC dimension \citep{vapnik1971uniform}]
Given a class of functions $\H$ from $\X$ to $\{0,1\}$, we say that a set $S\subseteq\X$ is \emph{shattered by $\H$} if the restriction of $\H$ to $S$ is the set of all function from $S$ to $\{0,1\}$.
The VC dimension of a hypothesis class $\mathcal{H}$, denoted $\VC(\mathcal{H})$, is the size $d$ of the largest set that can be shattered by $\mathcal{H}$.
If no such $d$ exists then $\VC(\mathcal{H})=\infty$.
\end{definition}

\paragraph{Computable learnability.}
The above notion of learnability only considers the size of a sample needed to ensure generalization; there are no computational limitations. 
\emph{Efficient} PAC learnability \citep{valiant1984theory} requires that the algorithm $\A$ run in time that is polynomial in the learning parameters, and that its output be polynomially evaluable. 
Sitting in the middle of these two viewpoints is \emph{computable} PAC (CPAC) learnability  \citep{agarwal2020learnability}, where we do not require computational efficiency, but rather that the learning algorithm and its output be computable. 
We additionally require that the hypothesis class $\H$ be computably representable (\citet{agarwal2020learnability}, Remark 4). 

\begin{definition}[Computable Representation of a Hypothesis class \citep{agarwal2020learnability}]
    A class of functions $\H$ is \emph{decidably representable (DR)} if there exists a decidable set of programs $\P$ such that the set of all functions computed by a program in $\P$ equals $\H$. We call it \emph{recursively enumerably} representable (RER) if there exists such a set of programs that is recursively enumerable.
\end{definition}

The following definition incorporates these requirements for the class that a learner outputs.

\begin{definition}[CPAC Learnability, \citep{agarwal2020learnability}]
     We say that a class $\H$ is (agnostic) CPAC learnable, if there is a computable (agnostic) PAC learner for $\H$ that outputs  total computable functions as predictors and uses a decidable (recursively enumerable) representation for these.
\end{definition}

\paragraph{Adversarial robustness.}
Let $\U:\X\rightarrow 2^\X$ be a \emph{perturbation type}, assigning each point $x\in\X$ a region $\U(x)$ accessible to an adversary, with the convention $x\in\U(x)$.
We define the robust risk as:
$$\R_\U(h;D)=\prob{(x,y)\sim D}{\exists z \in \U(x) \st h(z)\neq y}\enspace.$$

Similarly as in the standard setting, we will consider the robust loss function $\ell^\U:\H\times\X\times\{0,1\}\rightarrow\{0,1\}$ defined as $\ell^\U(h,x,y)=\mathbf{1}[\exists z\in\U(x) \st h(x)\neq y]$, and denote the empirical robust risk of predictor $h$ over a sample $S$ by $\R_\U(h;S)$.

We will here work with this \emph{stability} notion of robustness, while prior work has also explored other notions of robustness with emphasis on label correctness \citep{mahloujifar2019curse,gourdeau2021hardness, BhattacharjeeC21, chowdhury2022robustness}. 

\begin{definition}[Agnostic Robust PAC Learnability \citep{montasser2019vc}] 
    A hypothesis class $\H$ is $\U$-robustly PAC learnable in the agnostic setting if there exists a learner $\A$ and a function $m(\cdot,\cdot)$ such that for all $\epsilon,\delta \in (0,1)$ and for every distribution $D$, if the input to $\A$ is an i.i.d. sample $S$ from $D$ of size at least $m(\epsilon,\delta)$, then, with probability at least $(1-\delta)$ over the samples, the learner outputs a hypothesis $\A(S)$ with 
    $\R_\U(\A(S);D)\leq \underset{h\in\H}{\inf}\R_\U(h;D)+\epsilon\enspace.$
    The class is said to be $\U$-robustly PAC learnable in the realizable setting if the above holds under the condition that $\underset{h\in\H}{\inf}\R_\U(h;D)=0$.
\end{definition}

\paragraph{Computable robust learnability.}
We can straightforwardly adapt the definition of CPAC learnability to its robust counterpart:

\begin{definition}[Robust CPAC Learnability]
     We say that a class $\H$ is $\U$-robustly (agnostic) CPAC learnable, if there exists a $\U$-robust PAC learner for $\H$ that also satisfies the computability requirements of a CPAC learner.
\end{definition}

Similarly to the requirements for $\H$,
we can require the perturbation type $\U$ also be DR or RER:

\begin{definition}[Representations of Perturbation Types]
\label{def:pert-dr-rer}
    Let $\U:\X\rightarrow 2^\X$ be a perturbation type. 
    Then $\U$ is said to be \emph{decidably representable (or recursively enumerable)}  if the set $\{(x,z) \st x\in \X, z\in \U(x) \}$ is decidable (or recursicely enumerable, repectively). 
\end{definition}

Asking for perturbation types to be decidably representable is quite a natural requirement. 
 Indeed, we show in Appendix~\ref{app:example-u-dr} an example that uses a perturbation type $\U$ that is not DR and thus makes the impossibility of $\U$-robust CPAC learning trivial.

As opposed to the binary loss, which depends only on a simple comparison between two labels, and can thus always be evaluated for computable predictors, the robust loss is not always computably evaluable. 
In Section \ref{sec:oracle}, we explore the role of the following notion of loss computability.
\begin{definition}[Robust loss computably evaluable]
Given a hypothesis class $\H$, and perturbation type $\U$, we say that the robust loss for   $\U$ is \emph{computably evaluable on $\H$} if $\ell^\U:\H\times\X\times\{0,1\} \to \{0,1\}$ is a total computable function.
\end{definition}

\paragraph{Proper learnability and empirical risk minimization.} For all the above notions of learnability, we call the class $\H$ \emph{properly PAC/ CPAC/ $\U$-robustly PAC/ $\U$-robustly CPAC learnable} if there exists a learner that satisfies the corresponding definition of learnability and always outputs functions from~$\H$. 
A learner $\A$ is said to \emph{perform empirical risk minimization (ERM)}, or \emph{implement an $\erm_\H$ oracle}, 
if it always outputs a predictor from the hypothesis class of minimal empirical risk, that is for all $m\in \N$ and samples $S\in (\X\times \Y)^m$, 
\[
\A(S) \in\argmin_{h\in \H} \R(h , S)\enspace.
\]
Robust empirical risk minimization (RERM) is defined analogously for the robust risk $\R_\U$.
We now outline different types of RERM oracles that we will study throughout this work. 
A labelled sample $S$ is said to be \emph{robustly realizable} if $\min_{h\in \H}\R_{\U}(h;S)=0$, and \emph{robustly non-realizable} otherwise. 
We will distinguish between three cases:

   $\bullet$ {\bf weak-realizable $\rerm_{\H}^{\U}$ oracle:} for any robustly realizable sample $S$, the oracle outputs $h$ with $\R_{\U}(h;S)=0$. On robustly non-realizable samples such an oracle is allowed not to halt.
        
    $\bullet$ {\bf strong-realizable $\rerm_{\H}^{\U}$ oracle:} for any robustly realizable sample $S$, the oracle outputs $h$ with $\R_{\U}(h;S)=0$. For any robustly non-realizable sample, it halts and outputs ``not-robustly-realizable". 
    
    $\bullet$ {\bf agnostic $\rerm_{\H}^{\U}$ oracle:} the oracle performs RERM on both robustly realizable and robustly non-realizable samples.

    Note that, if the robust loss is computably evaluable, then the existence of a computable strong realizable $\rerm_{\H}^{\U}$ oracle is equivalent to the existence of an agnostic $\rerm_{\H}^{\U}$ oracle. Moreover, a strong-realizable oracle always implies the existence of an agnostic oracle.
    
\paragraph{Useful facts.} Recall the following sufficient condition for CPAC learnability, which we later use:

\begin{fact}[\cite{agarwal2020learnability, sterkenburg2022characterizations}]
\label{fact:ste22-vc-erm-computable-cpac}
    If $\VC(\H)$ is finite and 
    there exists a total computable function that implements an $\erm_\H$ oracle, then $\H$ is CPAC learnable.
\end{fact}
We now recall the following result by \citet{ashtiani2020black}, which outlines sufficient conditions for the proper robust learnability of a class. 
Throughout this text, we will use the notation of \citet{ashtiani2020black}, where the \emph{margin class} $\H^\U_{\text{mar}}:=\set{\mar^\U_h}_{h\in\H}$ consists of the margin sets $\mar^\U_h:=\set{x\in\X\given \exists z\in\U(x)\st h(x)\neq h(z)}$, i.e., given $\H$ and $\U$, the instances in $\X$ that incur a robust loss of 1 due to a perturbation that is labeled differently by $h$ (rather than due to mislabeling). 

\begin{fact}[Theorem 7 in \citep{ashtiani2020black}]
\label{fact:apu20-proper-rob-learning}
If both the VC dimension of a class $\H$ and the VC dimension of  $\H^\U_{\text{mar}}$ are finite, then $\H$ is properly (agnostically) $\U$-robustly PAC learnable.
\end{fact}
 
\section{Warm-up: Simple Sufficient Conditions for Robust CPAC Learnability}
\label{sec:warm-up}

We start by exhibiting conditions ensuring (proper) robust CPAC learnability in the agnostic setting.

\begin{fact}
\label{fact:suff-rob-cpac}
    Let hypothesis class $\H$ and perturbation type $\U$ be such that $\VC(\H)+\VC(\H_\mar^\U)<\infty$ and 
    there exists a total computable function that implements an agnostic $\rerm_{\H}^{\U}$ oracle.
    Then $\H$ is properly $\U$-robustly CPAC learnable.
    If there exists a computable function that implements a weak-realizable $\rerm_{\H}^{\U}$ oracle, then $\H$ is $\U$-robustly CPAC learnable in the realizable case.
\end{fact}

\begin{proof}
    By Fact~\ref{fact:apu20-proper-rob-learning}, $D:=\VC(\H)+\VC(\H_\mar^\U)<\infty$ implies that a finite sample of size $O\left(\frac{D\log D+\log 1/\delta}{\epsilon^2}\right)$ is sufficient to guarantee robust generalization w.r.t. $\U$, whenever a robust risk minimizer for a given sample is returned.
    Since robust ERM is computably implementable, we are done.
\end{proof}

\begin{fact}
    \label{fact:losseval+RER->realRERM}
    Let hypothesis class $\H$ be RER and perturbation type $\U$ be such that $\ell^{\U}$ is computably evaluable on $\H$.
    Then $\H$ admits a weak realizable RERM oracle.
\end{fact}

\begin{proof} 
   Let $S$ be a robustly realizable input sample. The following procedure computes $\rerm_{\H}^{\U}$. 
   Since the class is RER, we can iterate through the class $\H = (h_i)_{i\in\N}$ and for each $h_i$:\\
   $\bullet$ Compute $\R^{\U}(h_i,S)$, which is possible since the robust loss is computably evaluable;\\
   $\bullet$ Check whether $\R^{\U}(h_i,S)=0$. If so, halt and output $h_i$, otherwise continue.\\
   As we know that there exists $h \in \H$ with $R^{\U}(h,S)=0$, we know that the algorithm will halt.
\end{proof}

\begin{remark}
\label{rmk:suff-rob-cpac-realizable}
    {Combining Facts~\ref{fact:suff-rob-cpac} and~\ref{fact:losseval+RER->realRERM} shows that for a RER hypothesis class $\H$ and a perturbation type $\U$, the conditions $\VC(\H)+\VC(\H_\mar^\U)<\infty$ and $\ell^{\U}$ on $\H$ being computably evaluable are  sufficient to guarantee proper robust CPAC learnability in the robustly realizable case.}
\end{remark}

Now, we show that having access to a computable online learner (see Definition~15 in \citep{hasrati2023computable}) and a counterexample oracle, together with the margin class having finite VC dimension, is sufficient to guarantee robust learnability.
The counterexample oracle we will use is the Perfect Attack Oracle (PAO) of \citep{montasser2021adversarially}, which takes as input $h,x,y\in\H\times\X\times\{0,1\}$ and returns $z\in\U(x)$ such that $h(z)\neq y$, if such a counterexample exists.
The proof is included in Appendix~\ref{app:rob-cpac-algo}.

\begin{proposition}
\label{thm:cpac-from-c-online+finite-rob-loss+pao}
    Let $\H$ be computably online learnable, and let $\U$ be such that $\VC(\H_\mar^\U)<\infty$. 
    Then $\H$ is $\U$-robustly CPAC learnable with access to the PAO in the realizable setting.
\end{proposition}

\section{Relating CPAC and Robust CPAC Learning}
\label{sec:relating-cpac-rob-cpac}

In this section, we study the relationship between CPAC learnability, robust CPAC learnability and the computable evaluability of the robust loss.
We start in Section~\ref{sec:impossibility-results} with examples where we have CPAC learnability, but not robust CPAC learnability. The constructions in that section also show that the robust loss being computably evaluable is not sufficient for robust CPAC learnability.
In Section~\ref{sec:oracle}, we show that the robust loss being computably evaluable is in general also not necessary to ensure robust CPAC learning, which might be unexpected.

\subsection{Robust CPAC learnability is not implied by Robust PAC + CPAC learnability}
\label{sec:impossibility-results}

In this section, we first look at two examples of hypothesis classes and perturbation types where CPAC learnability and robust learnability do not imply robust CPAC learnability.
In both cases, the perturbation types are DR, in contrast to Example~\ref{ex:u-dr-necessary}.
Our first result holds in the robust agnostic case for general, improper learners, and the second, in the robust realizable case for proper learners.

\begin{theorem}
\label{thm:rob-cpac-imposs-agnostic}
    There exists a hypothesis class $\H$ and perturbation type $\U$ such that (i) $\H$ is properly CPAC learnable, (ii) $\U$ is decidably representable and (iii) $\H$ is properly $\U$-robustly PAC learnable, but $\H$ is not (improperly) $\U$-robustly CPAC learnable in the agnostic setting. %
\end{theorem}

\begin{proof}[Sketch]
    Fix a proof system for ﬁrst-order logic over a rich enough vocabulary that is sound and complete.
    Let $\set{\varphi_i}_{i\in\N}$ and $\set{\pi_j}_{j\in\N}$ be enumerations of all theorems and proofs, respectively.
    We define the following hypothesis class on $\N$:
    $$\H=\set{h_{a} \; : \; \forall i\in \N,\; h_{a}(2i)= \mathbf{1} [i\leq a] \wedge h_{a}(2i+1)= 1}_{a \in \N \cup\set{\infty}}\enspace.$$
    Thus, $\H$ is the the concept class where each function defines a threshold on even integers, and is the constant function 1 on odd integers. 
    
    The perturbation sets are defined as follows. For any $i\in\N$ we set:
    \begin{align*}
        &\U(6i) = \set{6i} \cup \set{2j+1}_{j\in\N} \enspace, \\
        &\U(6i+2) = \set{6i+2}  \enspace,\\
        &\U(6i+4) = \set{6i+2, 6i+4} \cup \set{2j+1\given \pi_j\text{ is a proof of theorem }\varphi_i}_{j\in\N} \enspace,\\
        &\U(2i+1) = \set{2i+1}  \enspace.
    \end{align*}
    The properties (i)-(iii) are fulfilled by this construction, as we will show in detail in the appendix.
    Furthermore, we note that the question of whether $\U(6i) \cap \U(6i+4)$ is empty is equivalent to whether there exists a proof for theorem $\varphi_i$, and thus is undecidable.
    We now show that this undecidable problem can be reduced to agnostically robustly CPAC learning $\H$ with respect to $\U$. To see this define the distributions $D_i$ on $\naturals \times \{0,1\}$ as follows:
    \begin{align*}
        D_i((6i,1)) = 1/2, \qquad
        D_i((6i+2,1)) = 1/6,\qquad
        D_i((6i+4,0)) = 1/3\enspace.
    \end{align*}
For a learner to succeed on all the $D_i$, it needs to decide  whether $\U(6i) \cap \U(6i +4) = \emptyset$ for all $i$.
\end{proof}

For more detail and the full proof we refer the reader to Appendix~\ref{app:proof-rob-cpac-imposs-agnostic}.

The result above holds in the improper agnostic setting. 
Next, we look at the robustly realizable setting, and show an impossibility result in case we require a \emph{proper} robust learning algorithm.

\begin{theorem}
\label{thm:proper-rob-cpac-imposs-realizable}
    There exists a hypothesis class $\H$ and perturbation type $\U$ such that (i) $\H$ is properly CPAC learnable, (ii) $\U$ is decidably representable, and (iii) $\H$ is properly $\U$-robustly PAC learnable, but not properly $\U$-robustly CPAC learnable in the realizable setting.
\end{theorem}

\begin{proof}
Let $(T_i)_{i\in \naturals}$ be an enumeration of Turing Machines. 
   We define the hypothesis class on $\N$:
    $$\H = \{h_{a,b,c} \;:\; \forall i \; (h_{a,b,c}(2i) = 1 \text{ iff } i \in\{a,b\}) \wedge (h_{a,b,c}(2i+1) = c)\}_{a,b\in \N, \;c\in \{0,1\}}\enspace,$$  
    i.e., $\H$ is the class  of functions that only maps at most two even numbers to 1 and are constant on the odd numbers.
    First note that $\VC(\H)=3$, as for any shattered set $X\subseteq \N$, we can have at most one odd integer in $X$, and at most two even ones. It is straightforward to show that any set of the form $\set{2i,2j,2k+1}$ for $i,j,k\in\N$ can be shattered. 
    This implies, by Theorem 6 in \citep{montasser2019vc}, that $\H$ is (improperly) robustly learnable for any perturbation type $\U$.
    We will argue below that, for the perturbation type we employ, $\H$ is in fact properly $\U$-robustly CPAC learnable.
    Second, as ERM is easily implementable for $\H$, we have that $\H$ is properly CPAC learnable. 
    
    We define the perturbation sets as follows:
    \begin{align*}
        &\U(2i) = \{2i\}\cup \{2j+1: T_i(i) \text{ halts within } j \text { steps } \} \enspace, \\
        &\U(2i+1) = \set{2i+1}  \enspace.
    \end{align*}

    Note that $\U$ is decidably representable, as for any two points $x_1,x_2$, the program outlined in Appendix~\ref{app:dr-algo-proper} decides whether $x_2 \in \U(x_1)$.

    We now show that for this $\U$, $\H$ is \emph{properly} $\U$-robustly learnable. To see this, we again use the terminology of \citet{ashtiani2020black}, and consider the margin class $\H^\U_\mar$, where each $h_{a,b,c}\in\H$ fixes a set $\mar_{h_{a,b,c}}^\U$, which we will denote by $\mar_{a,b,c}$ for brevity.
    Note that, from the property $k\subseteq\U(k)$ that holds for all $k\in\N$, only even integers can belong in some set $\mar_{a,b,c}$.
    Recall that $h_{a,b,c}$ maps at most two even numbers $a$ and $b$ to $1$ and all odd numbers to $c\in\{0,1\}$.
    
    We distinguish two cases:
    \begin{itemize}
        \item $c=0$: then $\mar_{a,b,0}=\set{2a\given T_a(a) \text{ halts}}\cup \set{2b\given T_b(b) \text{ halts}}$,
        \item $c=1$: then $\mar_{a,b,1}=\set{2i\given T_i(i)\text{ halts}}\setminus\set{2a,2b}$.
    \end{itemize}
    From this, we can see that shattered sets must only contain even integers $k$ such that $T_{k/2}(k/2)$ halts.
    We now argue that $\VC(\H^\U_\mar)=5$.
    Consider a set $X=\set{k_1,\dots,k_6}$ of size 6 such that $k_i$ is even and $T_{k_i/2}(k_i/2)$ halts for all $i=1,\dots,6$. 
    Then the subset $\set{k_1,k_2,k_3}$ cannot be part of the projection of $\H^\U_\mar$ onto $X$, and thus $\set{k_1,\dots,k_6}$ cannot be shattered. 
    Indeed, note that for $c = 0$, the margin set $\mar_{a,b,0}$ has size $2$ only for any $a,b\in\N$. 
    Thus $\set{k_1,k_2,k_3}\in \mar_{a,b,c}$ would imply that there exists $a,b\in\N$ such that $k_1,k_2,k_3\in\mar_{a,b,1}$. 
    However then, by definition, there are at most two of $k_4,k_5,k_6$ that are not in $\mar_{a,b,1}$, a contradiction.
    Finally, consider a set $X=\set{k_1,\dots,k_5}$ of size 5 such that $k_i$ is even and $T_{k_i/2}(k_i/2)$ halts for all $i=1,\dots,5$.
    Any subset $X'\subseteq X$ is contained in the projection of $\H^\U_\mar$ onto $X$: if $|X'|\geq 3$ there exists $a,b\in\N$ such that $X'\in\mar_{a,b,1}$, and otherwise if $|X'|\leq 2$ there exists $a,b\in\N$ such that $X'\in\mar_{a,b,0}$, and thus $X$ is shattered.
    
    Now, for any $i,k\in \N$, with $i\neq k$, we define a distribution $D_{i,k}$ on $\N\times\{0,1\}$ as follows:
    \begin{align*}
        D_{i,k}(2i,1)=D_{i,k}(2k,0)=1/2 \enspace.
    \end{align*}
    
    We will now show that $\H$ is not properly $\U$-robustly CPAC learnable with respect to the set of distributions $\set{D_{i,k}}_{i,k\in\N}$ in the realizable case.
    Note that robustly realizability here implies that at most one of the sets $\U(2i)$ or $\U(2k)$ is not a singleton.
    Indeed, if $\U(2i)=\set{2i}$ then $T_{i}(i)$ does not halt, whereas if $\set{2i}\subset\U(2i)$ (strict inequality) then $T_{i}(i)$ halts, and similarly for $2k$. 
    Hence, since we have $D_{i,k}(2i,1)=D_{i,k}(2k,0)=1/2$, to choose between $h_{i,i',0}$ and $h_{i,i',1}$ for some fixed $i'\in\N$ with $i'\neq k$ and $i'\neq i$ we need to know whether either of $\U(2i)$ or $\U(2k)$ is a singleton. 

    Now consider the following decision problem, which we term $\twohalt$.
    Given two Turing machines $T_i$ and $T_j$, the task is to correctly predict which run out of $T_i(i)$ and $T_j(j)$ halts if only one of them does. If both runs loop, the solver should still halt and produce an output; in case both halt the decider may run indefinitely or halt and produce some output. That is, we only require a correct decision if exactly one of $T_i(i)$ and $T_j(j)$ halts, and require the program to halt if both loop. We show below (see Lemma \ref{lem:twohalt_is_undecidable}) that no algorithm can satisfy these requirements.

    We now argue that a $\U$-robust proper CPAC learner can be used to solve $\twohalt$, and thus, by invoking Lemma \ref{lem:twohalt_is_undecidable}, which is stated and proved in Appendix~\ref{app:useful-proper-impossibility}, conclude that no such learner exists.  If a $\U$-robust proper CPAC learner for $\H$ existed, there is a sample size $M$ that suffices for robust loss at most $\epsilon = 1/3$ with probability $1-\delta = 2/3$ over all $M$ size samples from $D_{i,k}$. We then solve $\twohalt$ by running the learner on all $M$-length sequences over the points $(2i,1)$ and $(2k,0)$, and taking the majority for the label that the resulting predictor gives on the odd numbers.
    
    \end{proof}

\subsection{Robust CPAC learnability and its relationship to the computability of $\ell^\U$ and RERM}
\label{sec:oracle}

One of the obstacles with robust CPAC learning is that, even if a class consists of computable functions, the robust loss might not be pointwise computable. This is in contrast to the CPAC setting for which the binary loss is always computable for computable functions.
It is then natural to ask what relationship exists between the computability of the robust loss and the robust CPAC learnability of a class.
The following result demonstrates that the computability of the pointwise robust loss for all hypotheses in the class is not a necessary criterion for robust CPAC learnability.%

\begin{theorem}
\label{lemma:sep-cpac-rob-loss-uncomputable}
    There exist $\H$ and $\U$ such that (i) $\H$ is properly CPAC learnable, (ii) $\U$ is decidably representable, and (iii) $\H$ is properly $\U$-robustly CPAC learnable in the agnostic setting and admits a computable agnostic $\rerm_{\H}^{\U}$ oracle, but the function $\ell^\U$ is not computably evaluable on $\H$.
\end{theorem}

\begin{proof}[Sketch]
    This theorem is proven with the following construction.
    Let $\mathcal{X}$, $(\varphi_i)_{i \in \naturals}$ and $(\pi_i)_{i\in \naturals}$ be defined as in the construction for Theorem~\ref{thm:rob-cpac-imposs-agnostic}.
    Let $(i,0)$, represent the $i$-th formula and for any $i, j$ let $(i,j)$ represent the $j$-th proof (independently of $i$). 
    Let $\mathcal{H}=\{h_{a,b}\}_{a,b\in \naturals \cup \{0, \infty\}}$, where
    \[h_{a,b}(i,j) = 
        \begin{cases}
            1 &\text{ if } i < a \text{ or }(i = a \text{ and } j\leq b)\\
            0 &\text{ otherwise }
        \end{cases}
        \enspace.
    \]
    It is clear that $\VC(\H)=1$ and that ERM is computably implementable, so $\H$ is properly CPAC learnable.
    We define the perturbation regions as follows:
    \begin{align*}
        &\mathcal{U}(i,0) =\{(i,0)\}\cup \{(i,j): \pi_j \text{ proves } \varphi_i \}\enspace,\\
        &\mathcal{U}(i,j) = \{(i,0), (i,j)\} \enspace.
    \end{align*}
    It is also easy to verify that this perturbation type is DR.
    In the full proof, we show that in general there is no algorithm that evaluates the loss of each hypotheses $h_{a,b}$ on every point. However, there is an RER subclass $\H' = \{h_{a,\infty}: a \in \naturals\}\subset \H$ such that every hypothesis in $\H'$ is computably evaluable and such that on any distribution, $\H'$ contains an optimal (with respect to $\H$) hypothesis. It is thus possible to robustly CPAC learn $\H$ by performing RERM over $\H'$.
\end{proof}

The full proof is included in Appendix~\ref{app:proof-sep-cpac}. We next show that, even if the robust loss is computably evaluable on $\H$, and the class $\H$ is properly CPAC learnable and properly robustly PAC learnable, neither the existence of a robust ERM oracle nor robust CPAC learnability are implied.

\begin{theorem}
\label{thm:rob-loss-comp-but-no-proper-cpac-agnostic}
There exist a DR hypothesis class $\H$ and a DR perturbation type $\U$ such that (i) $\H$ is properly CPAC learnable, (ii) $\H$ is properly $\U$-robustly PAC learnable, and (iii) the robust loss is computably evaluable, but there is no strong realizable $\rerm_{\H}^{\U}$-oracle and $\H$ is not properly agnostically $\U$-robustly CPAC learnable.
\end{theorem}

\begin{proof}[Sketch]
We prove this theorem with the following construction.
Let the instance space be $\X= \naturals \times \naturals$.
Let the hypothesis class be $\H = \{h_{i,j}:i,j \in \naturals \}$, where 
\[h_{i,j}(x) = \begin{cases}
    1 & \text{ if } x=(i,k) \text{ with } k \leq j \\
    0 & \text{ otherwise }
\end{cases}\enspace,\]
i.e., $\H$ is the class of initial segments on $\{(i,j)\}_{j\in\N}$ for all $i\in\N$.
We set the perturbation types to:
\begin{align*}
    \U((i,0)) &= \{(i,0)\}\cup \{(i,k): T_i \text{ does not halt after } k \text{ steps on the empty input }\}\enspace,\\
    \U((i,j)) &= 
    \begin{cases} 
        \{(i,j), (i,0) \} & T_i \text{ does not halt after } j  \text{ steps on the empty input} \\ \{(i,j)\} & \text{ otherwise. } 
    \end{cases}
    \enspace
\end{align*}

The proof of the impossibility of agnostic robust CPAC learning and of the non-existence of a strong realizable RERM oracle relies on reductions from the Halting problem.
\end{proof}

The detailed proof shows that the construction in the sketch above fulfills all the requirements for proving the theorem and is included in Appendix~\ref{app:proof-rob-loss-comp-not-rcpac}.
Recall from Fact \ref{fact:losseval+RER->realRERM} that the class being RER (a weaker assumption than being DR) and the robust loss being computably evaluable imply the existence of a weak realizable  $\rerm_{\H}^{\U}$-oracle.
We therefore demonstrated a separation between the existence of RERM oracles in the robust realizable versus the agnostic case. Our proof also shows that the existence of weak realizable RERM oracles does not imply the existence of agnostic RERM orcalce. Furthermore, note that the existence of a strong realizable RERM oracle implies computable evaluablility of robust loss. Thus, the result of Theorem~\ref{lemma:sep-cpac-rob-loss-uncomputable} shows that the existence of a agnostic RERM oracle does not imply the existence of a strong realizable RERM oracle.

In the non-robust setting, proper \emph{strong} CPAC (SCPAC) learnability (which is CPAC learnability with the additional requirement that the sample complexity function $m(\cdot,\cdot,\cdot)$ be computable) implies the existence of an $\erm_{\H}$ oracle (\citet{sterkenburg2022characterizations}, Theorem~8). 
We finish this section by showing that a similar result holds in robust learning if the robust loss is computably evaluable.
The proof of Proposition~\ref{prop:proper->rerm} below is included in Appendix~\ref{app:proof-prop:proper->rerm}.

\begin{proposition} 
\label{prop:proper->rerm}
Let $\H$ be a hypothesis class and $\U$ a perturbation type such that $\ell^{\U}$ is computably evaluable. Then, if $\H$ is properly $\U$-robustly SCPAC learnable in the agnostic setting, $\H$ admits a computable agnostic $\rerm_{\H}^{\U}$ oracle; and if $\H$ is properly $\U$-robustly SCPAC learnable in the realizable setting, then $\H$ admits a computable weak realizable $\rerm_{\H}^{\U}$ oracle.
\end{proposition}

\section{Necessary Conditions for Robust CPAC Learning}
\label{sec:c-dim-u}

In this section, we define a complexity measure, denoted by $c$-$\dim_\U$, and prove its finiteness is implied by robust CPAC learnability for RER perturbation types.
We later study its relationship with the VC dimension, and show that it cannot characterize robust CPAC learnability.

\subsection{No-Free-Lunch Theorems}
\label{sec:nflt}

Prior work has shown that standard CPAC learnability is equivalent to the hypothesis class having finite \emph{effective VC dimension}, which can be thought of as a computable version of the VC dimension \citep{sterkenburg2022characterizations, delle2023find}. The effective VC dimension of $\H$ is the smallest $k$, such that there exists a computable function $w:\N^{k+1}\rightarrow\{0,1\}^{k+1}$ that, when given a set of domain points of size $k+1$, outputs a labelling that cannot be achieved by any $h\in\H$. The function $w$ is called a \emph{witness function}.
The complexity measure $c$-$\dim_\U(\H)$ that we introduce in this section is a computable version of the robust shattering dimension $\dim_\U(\H)$ which has been shown to lowerbound the robust sample complexity \citep{montasser2019vc}.

\begin{definition}[$\U$-robust shattering dimension \cite{montasser2019vc}]
\label{def:rob-shattering-dim}
    A set $X=\{x_i\}_{i=1}^m$ is $\U$-robustly shattered by $\H$ if there exists a set $\zdef$ with $x_i\in\U(\zin)\cap \U(\zip)$ for all $i=1,\dots,m$ such that for all $y_1\dots,y_m\in\{0,1\}$ there exists $h\in\H$ such that for all $i=1,\dots,m$ and $z'\in\U(z_i^{y_i})$ we have that $h(z')=y_i$.
    The \emph{robust shattering dimension of $\H$}, denoted $\dim_\U(\H)$, is the largest integer $k$ such that there exists a set of size $k$ that is $\U$-robustly shattered.
    If arbitrarily large $\U$-robustly shattered sets exist, then $\dim_\U(\H)=\infty$.%
\end{definition}

Note that, for $X$ to be robustly shattered, for all $i\neq j$ the regions $\U(\zip)$ and $\U(\zjn)$ must be disjoint. 
Now, contrary to the standard notion of shattering, there are two factors at play in a set being robustly shattered.
The first one is whether a set is robustly \emph{``shatterable''}, namely whether, by definition of $\U$--and without considering $\H$--the set $X$ \emph{could} be robustly shattered by an arbitrary set of functions (which holds trivially in the standard case). 
The second factor, if robust shatterability holds, is as in the standard case: whether there does in fact exists a set of hypotheses from $\H$ that can realize the robust shattering.
To highlight this distinction, we introduce the following property:

\begin{definition}[Robust shatterability]
\label{def:rob-shatterability}
    A set $X=\set{x_1,\dots,x_m}\subseteq\X$ is called \emph{$\U$-robustly shatterable} if there exists a set $\zdef\subseteq\X$ with $x_i\in\U(\zin)\cap \U(\zip)$ and for all $i\neq j$ the regions $\U(\zip)$ and $\U(\zjn)$ are disjoint.  
    In this case, $Z$ is said to \emph{admit a $\U$-robustly shatterable set}.
\end{definition}
We now define the computable version of Definition~\ref{def:rob-shattering-dim}, which also uses the notion of a witness function.
If $\dim_\U(\H)<k$ then for all $m\geq k$, for all $X=\{x_i\}_{i=1}^m$ and for all $\zdef$ with $x_i\in\U(\zin)\cap \U(\zip)$ there exists a labelling 
that is \emph{not} $\U$-robustly realizable by $\H$.
Thus: 

\begin{definition}[Computable robust shattering dimension]
\label{def:c-dim-u}
    A \emph{$k$-witness of  $\U$-robust shattering dimension} for $\H$ is a function $w:\N^{2(k+1)}\rightarrow\{0,1\}^{k+1}$ such that when given a set $Z=\{(\zin,\zip)\}_{i=1}^{k+1}$ that admits a robustly shatterable set, 
    it outputs $y_1\dots,y_{k+1}\in\{0,1\}$ such that for all $h\in\H$ there exist $i=1,\dots,m$ and $z'\in\U(z_i^{y_i})$ with $h(z')\neq y_i$. 
    The \emph{computable robust shattering dimension of $\H$}, denoted $\text{c-}\dim_\U(\H)$, is the smallest integer $k$ such that there exists a \emph{computable} $k$-witness of robust shattering dimension.
    If no such $k$ exists, $\text{c-}\dim_\U(\H)=\infty$.
\end{definition}

We note that, by definition, if $\U$ does not admit a shatterable set of size $k$, then any function is a $k$-witness.
Towards showing that finiteness of the computable robust shattering dimension is a necessary condition for $\U$-robust CPAC learnability, we first formulate a No-Free-Lunch theorem for robust learnability, leaving aside computability considerations for the time being.

\begin{lemma}[Robust No-Free-Lunch Theorem]
\label{lemma:rob-nflt}
    Let $\X$ be the domain and $\U:\X\rightarrow 2^\X$ a perturbation type. 
    Let $\A$ be any computable learner.
    Let $m$ be a number smaller than $|\X|/4$ representing the training set size, with the property that there exist $4m$ points $\zdef[2m]$ such that $Z$ admits a $\U$-robustly shatterable set.
    Then there exists a distribution $D$ over $\X\times\{0,1\}$ such that
    (i) there exists a function $f:\X\rightarrow\{0,1\}$ such that $\R_\U(f;D)=0$; and
    (ii) with probability at least $1/7$ over $S\sim D^m$, $\R_\U(\A(S);D)\geq 1/8$. 
\end{lemma}

The proof of Lemma~\ref{lemma:rob-nflt} is included in Appendix~\ref{app:nflt}.
We comment on the robust shatterability in the impossibility results from Section~\ref{sec:relating-cpac-rob-cpac} in Appendix~\ref{app:comment-impossibility}. %

\begin{remark}
    First note that we work with $\X = \N$, so the condition $m\leq |\X|/4$ is immediately satisfied for any $m$.
    Second, note that the robust shatterability condition in the theorem, which depends on the perturbation type $\U$, is satisfied if the regions $\U(x)$ are, e.g., balls of finite radii. 
\end{remark}

We are now ready to state the computable version of Lemma~\ref{lemma:rob-nflt}, which involves not only the existence of a distribution and function that are robustly realizable on which the learner performs poorly, but also the ability to computably find such a pair:

\begin{lemma}[Computable Robust No-Free-Lunch Theorem]
\label{lemma:c-rob-nflt}
    Fix instance space $\X$ and perturbation type $\U:\X\rightarrow2^\X$.
    Let $\A$ be any computable learner and furthermore suppose that the perturbation type $\U$ is recursively enumerably representable (RER). 
    Then for any $m\in\N$, any domain $\X$ of size at least $4m$ and any subset $\zdef[2m]$ that admits a $\U$-robustly shatterable set, we can computably find a function $f:\X\rightarrow\{0,1\}$ that is computable on $Z$ such that 
    \begin{equation}
    \label{eqn:c-rob-nflt}
        \eval{S\sim D^m}{R_\U(\A(S);D)}\geq 1/4 \qquad \text{and} \qquad \prob{S\sim D^m}{R_\U(\A(S);D)\geq 1/8}\geq 1/7\enspace, 
    \end{equation}
    where $D$ is the uniform distribution on $\set{(z_i^{y_i},y_i)}_{i=1}^{2m}$ for some $y\in\{0,1\}^{2m}$ such that for all $i=1,\dots,2m$ and for all $z'\in\U(z_i^{y_i})$ we have that $f(z')=y_i$ .
\end{lemma}
The proof of Lemma~\ref{lemma:c-rob-nflt} follows a similar argument as the proof of the computable  No-Free-Lunch Theorem shown in \cite{agarwal2020learnability}, but the computation of the robust loss differs:
\begin{proof}
We first computably find a shatterable set $X = \{x_1, x_2, \dots, x_{2m}\}$
corresponding to the set of pairs $\zdef[2m]$. 
That is $x_i \in \U(\zin)\cap\U(\zip)$ for all $i\in[2m]$. Since the sets $\U(\zin)$ and $\U(\zip)$ are recursively enumerable, we can computably find such an $x_i$ in their intersection for each $i$.
We simply enumerate the instances from  $\U(\zin)$ and $\U(\zip)$ by alternating between the two sets until we find an instance that belongs to both enumerations. 
Because we are guaranteed that $Z$ admits a $\U$-robustly shatterable set, this process is guaranteed to halt.

The existence of a function $f$ and distribution $D$ satisfying Equation~\ref{eqn:c-rob-nflt} follows directly from Lemma~\ref{lemma:rob-nflt}, so it remains to show that, using $\A$ we can computably find such an $f$ and $D$.
Note that the construction of the pairs $(f_j,D_j)$ in Lemma~\ref{lemma:rob-nflt} corresponds to pairs of $\U$-robustly realizable distributions. 
There are $T=2^{2m}$ such pairs, as putting mass on $z_i^{y_i}$ implies not putting mass on $z_i^{\neg y_i}$.
Second, for a given $(f_j,D_j)$ and a fixed sample size $m$, there are $k=(2m)^{m}$ sequences $\{S_l^j \}_{l=1}^k $ of $m$ instances drawn from $D_j$, each equally likely as $D_j$ is uniform on its support. 
We now run $\A$ on all these sequences, and compute a lower bound on the expected robust loss of $\A$ on $(f_j,D_j)$ by employing the labels of the learner's output on the shatterable set $X$: 
    \begin{align*}
        \eval{S\sim D_j^m}{\R_\U(\A(S);D_j)}
        =\frac{1}{k}\sum_{l=1}^k \R_\U(\A(S^j_l);D_j)
        & =\frac{1}{k}\sum_{l=1}^k \frac{1}{2m}\sum_{i=1}^{2m}\ell^\U\left(\A(S^j_l),z_i^{(j)},y_i^{(j)}\right)\\
        & \geq \frac{1}{k}\sum_{l=1}^k \frac{1}{2m}\sum_{i=1}^{2m}\indct{\A(S^j_l)(x_i) \neq y_i^{(j)}}\enspace.
    \end{align*}
   As argued in the proof of Lemma~\ref{lemma:rob-nflt}, there exist $f_j, D_j$ such that $$\frac{1}{k}\sum_{l=1}^k \frac{1}{2m}\sum_{i=1}^{2m}\indct{\A(S^j_l)(x_i) \neq y_i^{(j)}}
   \geq 1/8\enspace.$$ Furthermore, we note that the for every $j$, the set of sequences $\{S_l^j\}_{j=1}^{k}$ can be computably generated and that the expression $\frac{1}{k}\sum_{l=1}^k \frac{1}{2m}\sum_{i=1}^{2m}\indct{\A(S^j_l)(x_i) \neq y_i^{(j)}}$ can be computably evaluated. Thus by iterating through $j=1, \dots, T$ until finding $j$ that satisfies $\frac{1}{k}\sum_{l=1}^k \frac{1}{2m}\sum_{i=1}^{2m}\indct{\A(S^j_l)(x_i) \neq y_i^{(j)}}
   \geq 1/8$, we can computably find $(f_j,D_j)$ with  $\eval{S\sim D_j^m}{\R_\U(\A(S);(D_j)} \geq 1/8$. We further note that the corresponding labelling $f_j$ can be computably evaluated on all elements of $Z$ as required.
\end{proof}

Note that if the perturbation sets are DR, then the function $f$ above is total computable on $\X$.
We now state the main result of this section:

\begin{theorem}
\label{thm:nec-rob-cpac}
    Let $\H$ be (improperly) $\U$-robustly CPAC learnable and suppose the perturbation type $\U$ is RER.
    Then the computable $\U$-robust shattering dimension $\text{c-}\dim_\U(\H)$ of $\H$ is finite, i.e., $\H$ admits a computable $k$-witness of $\U$-robust shattering dimension for some $k\in\N$.
\end{theorem}

The proof, included in Appendix~\ref{app:proof-nec-rob-cpac}, follows the reasoning of Lemma~9 in \cite{sterkenburg2022characterizations}.

\subsection{Relationship between the effective VC dimension and $c\text{-}\dim_\U$}

It has been shown that the dimension $\dim_\U(\H)$ yields a lower bound on the sample complexity for $\U$-robust learning \citep{montasser2019vc}. However,  it has remained an open question whether this dimension also provides an upper bound on the sample complexity, thus whether it actually characterizes $\U$-robust learnability. Here we show that a corresponding conjecture in the CPAC setting is not true. We first relate the effective VC dimension, which we will identify as $\cvc$ to denote computability, to the computable robust shattering dimension.

\begin{lemma}
\label{lemma:c-dim-u-upper-bound-c-vc}
    For any class $\H$ and perturbation type $\U$ admitting a recursively enumerable representation, we have $\cdim(\H)\leq \cvc(\H)$.
\end{lemma}

\begin{proof}
    Let $\H$ have finite effective VC dimension $k$. 
    Then there exist a computable function $w:\N^{k+1}\rightarrow\{0,1\}^{k+1}$ such that $w$ is a $k$-witness of VC dimension for $\H$.
    Now, let $\zdef[k+1]$ admit a robustly shatterable set.
    Because $\U$ is RER, we can find a robustly shatterable set $X$ of size $k+1$ for $Z$. 
    It suffices to output the labelling $w(X)$, which is not achievable by any $h\in\H$, and thus not robustly achievable by any $h\in\H$ as well, and so $c\text{-}\dim_\U(\H)\leq k$, as required.
\end{proof}

We know that CPAC learnability of a class $\H$ implies $c\text{-}\VC(\H) < \infty$ \citep{sterkenburg2022characterizations}, and thus, by the above lemma, CPAC learnability implies 
$c\text{-}\dim_\U(\H)\leq c\text{-}\VC(\H) <\infty$. On the other hand, we have shown in this work that there exist $\H$, $\U$ such that  CPAC learnability does not imply $\U$-robust CPAC learnability (see Theorem \ref{thm:rob-cpac-imposs-agnostic}).
 Thus there are classes that are not $\U$-robust CPAC learnable while having finite effective robust shattering dimension, meaning that this dimension does not characterize robust CPAC learnability:

\begin{corollary}
\label{cor:c-dim-u-not-char-cpac}
The dimension $ c\text{-}\dim_\U(\H)$ does not characterize $\U$-robust CPAC learnability.
\end{corollary}

\section{Conclusion}

We have initiated the study of robust computable PAC learning, and provided a formal framework for its analysis. 
We showed sufficient conditions that enable robust CPAC learnability, as well as showed that CPAC learnability and robust learnability are not in themselves sufficient to guarantee robust CPAC lernability.
We also exhibited a counterintuitive relationship between the computability of the robust loss and robust CPAC learnability.
This is of particular interest, as the evaluability of the robust loss has often implicitly been used in the literature on the theory of robust learning, or potential issues on the evaluability of the robust loss have been circumvented by the use of oracles. 
We finished by studying the role of the computable robust shattering dimension in robust CPAC learnability, showing that its finiteness is a necessary but not sufficient condition. 

\section*{Acknowledgements}
{Pascale Gourdeau has been supported by a Vector Postdoctoral Fellowship and an NSERC Postdoctoral Fellowship. Tosca Lechner has been supported by a Vector Research Grant and a Waterloo Apple PhD Fellowship. Ruth Urner is also an Affiliate Faculty Member at Toronto's Vector Institute, and acknowldeges funding through an NSERC Discovery grant.}

\bibliographystyle{plainnat}
\bibliography{robust-cpac}

\newpage
\appendix

\section{Proofs and Remarks from Section~\ref{sec:warm-up}}
\label{app:proofs-warm-up}

\subsection{Proof of Theorem \ref{thm:cpac-from-c-online+finite-rob-loss+pao}}
\label{app:rob-cpac-algo}

\begin{proof}
    We will use the computable online algorithm $\A$ as a black-box. 
    Let $D:=\VC(\H)+\VC(\H_\mar^\U)$ and note that online learnability implies that $\VC(\H)\leq\Lit(\H)<\infty$, and so $D<\infty$.
    
    Let $S=\set{(x_i,y_i)}_{i=1}^m$ be drawn i.i.d. from an arbitrary robustly-realizable distribution $D$ on $\X\times\{0,1\}$.  
    We outline below a robust CPAC algorithm, which is essentially the $\mathsf{CycleRobust}$ algorithm of \citep{montasser2021adversarially} with the online algorithm $\A$ being used, implements robust ERM on $S$:
    
    \begin{algorithm}[H]
    \caption{$\U$-robust CPAC algorithm}
    \KwIn{$S=\set{(x_i,y_i)}_{i=1}^m$, computable online algorithm $\A$ for $\H$}
    \KwOut{$h\in\H$ such that $h$ achieves zero robust loss on $S$}

    $K\gets 0$, $h\gets \A(\emptyset)$\;
    \While{$K\leq \Lit(\H)$}{
        \For{$i=1,\dots,m$}{
            \uIf(\Comment{robust loss is 1, get counterexample}){$(z,y)\gets \mathsf{PAO}(h,x_i,y_i)$}{
                $K++$ \; \Comment{Update mistake count}\\
                $h\gets\A((z,y))$ \Comment{Update hypothesis}\\
                Break \Comment{Exit for-loop}
            }
            \uElseIf{$i=m$}{
                return $h$ \Comment{$h$ is robustly consistent on $S$}
            }\Comment{No counterexample $\Rightarrow$ go to next index}
        }
    }
    return $h$
    \end{algorithm}
 
    Note that, because of robust realizability, we are guaranteed that $\A$ makes at most $\Lit(\H)$ mistakes, after which $h$ will be $\U$-robustly consistent on $S$. 
    The program above makes a finite number of calls to $\A$, and thus is computable by the fact that $\A$ is itself a computatble online learner.
    By Fact~\ref{fact:suff-rob-cpac}, we are done.
\end{proof}

\subsection{Decidable Representation of Perturbation Types}
\label{app:example-u-dr}

\begin{example}
\label{ex:u-dr-necessary}
    We exhibit $\H$, $\U$ such that $\H$ is CPAC learnable and $\U$-robustly PAC learnable but $\H$ is not (improperly) $\U$-robustly CPAC learnable. 
    In this example, $\U$ is not DR. 
    The hypothesis class is as follows:
    $$\H:=\set{h_{a,b}(i)=\mathbf{1}[i\in\set{a,b}]}\enspace.$$
    Clearly $\VC(\H)=2$ and ERM is computably implementable, so by Fact~\ref{fact:ste22-vc-erm-computable-cpac}, $\H$ is CPAC learnable.
    Moreover, $\H$ having finite VC dimension implies that it is $\U$-robustly PAC learnable \citep{montasser2019vc}.
    
    We now define the perturbation type:
    $$\U(2i)=\{2i\}\cup\{2i+1\given T_i \text{ halts on the empty word}\}\enspace,$$
    $$\U(2i+1)=\{2i+1\}\cup\{2i\given T_i \text{ halts on the empty word}\}\enspace.$$
    
    We now define distributions $D_i$:
    $$D_i(2i,1)=D_i(2i+1,0)=1/2\enspace.$$

    Clearly, the labelling $l(2_i)=1$ and $l(2i+1)=0$ is optimal (and gives zero robust risk) if and only if $T_i$ does not halt on the empty word. 
    Otherwise, if $T_i$ halts, $l$ incurs a robust risk of 1, and any function satisfying $h(2i)=h(2i+1)$ is a robust risk minimizer incurring a robust risk of $1/2$.

    If $\H$ were (improperly) $\U$-robustly CPAC learnable, there would be a robust learning algorithm $\A$ for $\H$ with sample complexity $m(\epsilon,\delta)$. 
    Letting $M\in\N$ being an upper bound on $m(1/3,1/2)$, then with probability at least $2/3$, $\A$ returns a hypothesis with robust risk within $1/3<1/2$ of the robust risk minimizer in $\H$ when run on a sample of size $M$.
    Then, running $\A$ on all $2^M$ sequences of length $M$ with elements in $\{2i,2i+1\}$ (which are equally probable), we end up with optimal hypotheses on two thirds of the runs.
    Checking, for each hypothesis $h$, whether $h(2i)=h(2i+1)$ and taking the majority vote over all sequences, we get a decider for whether $T_i$ halts on the empty word.
    And so $\H$ is not (improperly) $\U$-robustly CPAC learnable.
    Note here that if we had a way to know whether $k\in\U(2i)$, i.e., if $\U$ was DR, we would immediately be able to implement robust ERM for samples from the distributions defined above.
    Not requiring perturbation regions to be DR thus makes deriving impossibility results for robust CPAC learning relatively trivial.  
\end{example}

\section{Proofs and Algorithms from Section~\ref{sec:impossibility-results}}
\label{app:proofs-impossibility-results}

\subsection{Proof of Theorem~\ref{thm:rob-cpac-imposs-agnostic}}
\label{app:proof-rob-cpac-imposs-agnostic}
\begin{proof}
    Fix a proof system for ﬁrst-order logic over a rich enough vocabulary that is sound and complete.
    Let $\set{\varphi_i}_{i\in\N}$ and $\set{\pi_j}_{j\in\N}$ be enumerations of all theorems and proofs, respectively.
    We define the following hypothesis class on $\N$:
    $$\H=\set{h_{a} \; : \; \forall i\in \N,\; h_{a}(2i)= \mathbf{1} [i\leq a] \wedge h_{a}(2i+1)= 1}_{a \in \N \cup\set{\infty}}\enspace.$$
    Thus, $\H$ is the the concept class where each function defines a threshold on even integers, and is the constant function 1 on odd integers. 

    First, to show (i), note that $\VC(\H)=1$, as for any set $X$ of two even integers, the labelling $h(x)=0$, $h(z)=1$ cannot be realized on two even integers $x < z$, and odd integers can only have labelling 1, and thus cannot be in a shattered set.
    As ERM is easily implementable for $\H$, we have that $\H$ is properly CPAC learnable by Fact~\ref{fact:ste22-vc-erm-computable-cpac}.

    Now, we define the perturbation sets as follows. For any $i\in\N$ we set:
    \begin{align*}
        &\U(6i) = \set{6i} \cup \set{2j+1}_{j\in\N} \enspace, \\
        &\U(6i+2) = \set{6i+2}  \enspace,\\
        &\U(6i+4) = \set{6i+2, 6i+4} \cup \set{2j+1\given \pi_j\text{ is a proof of theorem }\varphi_i}_{j\in\N} \enspace,\\
        &\U(2i+1) = \set{2i+1}  \enspace.\\
    \end{align*}
    Note that the question of whether $\U(6i) \cap \U(6i+4)$ is empty is equivalent to whether there exists a proof for theorem $\varphi_i$, and thus is undecidable.
  
    It is straightforward to show (ii), as for every integer $i\in\N$, we can represent $\U(6i+4)$ with the program $P_{6i+4}$, which we have included below.

    \begin{algorithm}[H]
        \label{alg:u-dec-repr}
        \caption{Program $P_{6i+4}$ to decide whether $x\in\U(6i+4)$ }
            \KwIn{$x,i\in \N$} 
            \KwOut{Whether $x\in\U(6i+1)$} 
            \uIf {$x\in\set{6i+2,6i+4}$}{
                output yes
            }
             \uElseIf {$x$ is even}{
                output no
             }
             \Else{
                $j\leftarrow \frac{x-1}{2}$\;
                \uIf{$\pi_j$ proves $\varphi_i$}{
                    output yes
                }
                \Else{
                    output no
                }
             }
    \end{algorithm}

    Now, for all other $k$, deciding whether $x\in\U(k)$ is also easily representable as a program $P_k$ (we can just assume $\U(x)=\set{x}$ if there does not exists $i\in\N$ such that $k\in\set{6i,6i+2,6i+4}$, in order for $\U$ to be well-defined).

    To show (iii),  we follow the notation of \citep{ashtiani2020black} and show that $\VC(\H^\U_{\mar})=1$, which implies proper robust learnability by Fact~\ref{fact:suff-rob-cpac} (also  Fact~\ref{fact:apu20-proper-rob-learning}).
    Recall that the class $\H^\U_{\text{mar}}$ consists of the sets $\mar^\U_h:=\set{x\in\X\given \exists z\in\U(x)\st h(x)\neq h(z)}$, defined for each $h\in\H$.
    Now, fix $a\in\N$, which induces the function $h_a$ and the set $\mar_{h_a}^\U$, which we will denote by $\mar_a$ for simplicity.
    Because of the condition $h(x)\neq h(z)$ in the definition of $\mar_a$ and the property that  $\set{k}\subseteq\U(k)$, the only integers $k\in\N$ that can belong to $\mar_a$ are those for which $\set{k}$ is a proper subset of $\U(k)$, i.e., those for which there exists $i\in\N$ such that $k=6i$ or $k=6i+4$.
    We can see that 
    \begin{align*}
    \mar_a=&\set{6i \given 6i>2a}_{i\in\N}\cup\set{6i+4 \given \varphi_i \text{ is a tautology} \wedge 6i+4>2a}_{i\in\N} \\
    &\cup\set{6i+4 \given \varphi_i \text{ not a tautology} \wedge 6i+2=2a}_{i\in\N}\enspace,    
    \end{align*}
    as these are precisely the instances which incur a robust loss of 1 with respect to $h_a$:
    \begin{itemize}
        \item If $k=6i$ then since $\U(6i) = \set{6i} \cup \set{2j+1}_{j\in\N}$, we have that $k\in\mar_a$ iff $h_a(k)=0$,
        \item If $k=6i+4$ and $\varphi_i$ is a tautology, then since $$\U(6i+4) = \set{6i+2, 6i+4} \cup \set{2j+1\given \pi_j\text{ is a proof of theorem }\varphi_i}_{j\in\N}\enspace,$$ we have that $k\in\mar_a$ iff $h_a(k)=0$ (we cannot have $h_a(6i+4)=1$ while $h_a(6i+2)=0$), 
        \item Else, if  $k=6i+4$ and $\varphi_i$ is not a tautology, $\U(6i+4) = \set{6i+2, 6i+4}$, and $k\in\mar_a$ iff $h(6i+2)\neq h(6i+4)$ iff $6i+2=2a$.
    \end{itemize}
    Now, it suffices to consider the subsets of $X=\set{6i,6i+4}_{i\in\N}$ to find the largest shattered set, as these are the only instances that can belong to some set $\mar_a$.
    We partition $X$ into two parts $P_1$ and $P_2$, where $P_1=\set{6i}_{i\in\N}\cup\set{6i+4\given \varphi_i \text{ is a tautology}}_{i\in\N}$ and $P_2=\set{6i+4\given \varphi_i \text{ is not a tautology}}_{i\in\N}$.
    On the one hand, $\set{\mar_a}_{a\in\N}$ restricted to $P_1$ is simply thresholds on $P_1$, and so at most one element from $P_1$ can be in a shattered set. 
    On the other hand, $\set{\mar_a}_{a\in\N}$ restricted to $P_2$ is the class of singletons, so again at most one element from $P_2$ can be in a shattered set.  
    Thus, we have established that $\VC(\H^\U_\mar)\leq2$.
    To see that $\VC(\H^\U_\mar)=1$ consider arbitrary $k_1\in P_1$ and $k_2=6i+4\in P_2$ for some $i,i'\in\N$. 
    It is easy to see that $\set{k_1,k_2}$ or $\set{k_2}$ won't be in the projection of $\H^\U_\mar$ onto $\set{k_1,k_2}$.
    Indeed, $k_2$ being in a set in the projection implies that there is a unique $\mar_a$ such that $k_2\in\mar_a$, and $k_1$ is either in $\mar_a$ or not.
 
    Finally, to show that $\H$ not (improperly) $\U$-robustly CPAC learnable in the agnostic setting we define a set of distributions for which no computable learner for $\H$ succeeds. For each $i\in \N$ we define a distribution $D_i$ on $\N\times\{0,1\}$ as follows:
    \begin{align*}
        &D_i( (6i,1) ) = 1/2 \enspace, \\
        &D_i( (6i+2,1) ) = 1/6 \enspace,\\
        &D_i( (6i+4,0) ) = 1/3 \enspace.
    \end{align*}

  To determine optimal predictors on these distributions, we distinguish two cases:
    \begin{enumerate}
        \item $\U(6i) \cap \U(6i+4)=\emptyset$ : this implies that $\U(6i+4)=\set{6i+2, 6i+4}$. Since $\U(6i+2)\cap \U(6i+4)\neq \emptyset$ and these points have different labels under distribution $D_i$, no predictor can achieve robust loss $0$ on both points simultaneously. Thus no predictor can achieve robust risk less than $1/6$. Note that $h_{3i,\infty}$ has robust risk $1/6$ on $D_i$, and is therefore an optimal predictor on $D_i$.  Furthermore, note that any optimal predictor $h$ on $D_i$ must satisfy $h(6i+2)=h(6i+4)=0$.
        \item $\U(6i) \cap \U(6i+4)\neq\emptyset$ : in this case, since $6i$ and $6i+4$ have different labels under distribution $D_i$ and intersecting perturbation sets, no predictor can achieve robust loss $0$ on both points and thus no predictor can achieve robust loss less than $1/3$ on $D_i$.
        Note that $h_{\infty,\infty}$ (the constant function 1) has robust risk $1/3$ on $D_i$ and is thus a robust risk minimizer in this case. Furthermore, any optimal predictor $h$ on $D_i$ must satisfy $h(6i) = h(6i+2) = 1$, to not incur additional robust loss.
       
    \end{enumerate}

   We thus established that we have $h(6i+2) = 1$ for an optimal predictor on $D_i$ (and any predictor that has robust loss less than $1/6$ more than the optimal) if and only if $\U(6i) \cap \U(6i+4)\neq\emptyset$ and thus if and only if $\varphi_i$ has a proof.

  We now argue that for any $i\in\N$, a $\U$-robust CPAC learner $\A$  for $\H$, proper or not, can be used to determine whether  $\U(6i) \cap \U(6i+4)=\emptyset$, and thus can be used to decide if $\varphi_i$ is a tautology or not: we let $\A$ be a robust computable learner for $\H$ and let $M\in\N$ be an upper bound for learning $\H$ with $\epsilon = 1/7 < 1/6$ and $\delta = 1/3$. Running $\A$ on all sequences of length $M$ over the set $\set{(6i,1),(6i+2,1),(6i+4,0)}$ and evaluating the resulting predictors has to yield a $1/7$-close to optimal hypothesis on two thirds of the sample. Thus taking the majority vote over the predictions on $(6i+2)$ is a decider for whether $\varphi_i$ is a tautology or not.
\end{proof}
 
\subsection{Algorithm for the Decidable Representation in Theorem~\ref{thm:proper-rob-cpac-imposs-realizable} }
\label{app:dr-algo-proper}

    \begin{algorithm}[H]
        \label{alg:u-dec-repr-turing}
        \caption{Program to decide whether $x_2 \in \U(x_1)$}
            \KwIn{$x_1,x_2\in \N$}
            \KwOut{Whether $x_2 \in \U(x_1)$} 
            \uIf {$x_1$ is odd}{
                \uIf{$x_2=x_1$}{
                    output $x_2\in \U(x_1)$
                }
                \Else {
                    output $x_2 \notin \U(x_1)$
                }
            }
            \Else (\Comment{$x_1$ is even}){
                \uIf {$x_2$ is odd}{
                    run $T_{x_1}(x_1)$ for at most $\frac{x_2-1}2$ steps\\
                    \uIf{$T_{x_1}(x_1)$ halted}{
                        output $x_2 \in \U(x_1)$
                    }
                    \Else (\Comment{$T_{x_1}(x_1)$ is still running after $\frac{x_2-1}2$ steps}){
                        output $x_2 \notin \U(x_1)$
                    }
                }
                \Else (\Comment{$x_2$ is even}){
                    \uIf{$x_2=x_1$}{
                        output $x_2\in \U(x_1)$
                    }
                    \Else{
                        output $x_2\notin \U(x_1)$
                    }
                }
            }
    \end{algorithm}

\subsection{Useful Lemmas for Theorem~\ref{thm:proper-rob-cpac-imposs-realizable}}
\label{app:useful-proper-impossibility}
\begin{lemma}\label{lem:twohalt_is_undecidable}
There is no program that solves the problem $\twohalt$, that is, no program with the following behavior:
given a pair of indices of Turing Machines  $(i,j)$ as input:
\begin{itemize}
    \item if $T_i(i)$ halts and $T_j(j)$ loops, output $1$
    \item if $T_i(i)$ loops and $T_j(j)$ halts, output $2$
    \item if both  $T_i(i)$ and $T_j(j)$ loop, halt and output $1$ or output $2$
    \item if both $T_i(i)$ and $T_j(j)$ halt, output $1$ or output $2$ or loop
\end{itemize}
\end{lemma}

Before proving Lemma~\ref{lem:twohalt_is_undecidable}, let us state and prove the following result, which will be used in the proof of Lemma~\ref{lem:twohalt_is_undecidable}.

\begin{lemma}\label{lem:toscas_two_fold_recursion}
    For any two 2-place total computable functions $f_1$, $f_2$ there are indices $c_1,c_2$, such that
    $T_{c_1} \equiv T_{f_1(c_1,c_2)}$ and
    $T_{c_2} \equiv T_{f_2(c_1,c_2)}$
\end{lemma}
\begin{proof}
    We start by defining the following Turing Machines:
   \[T_a(x_1,x_2,y_1,y_2) = \begin{cases}
       T_{T_{x_1}(x_1,x_2)}(y_1,y_2) & \text{ if } T_{x_1}(x_1,x_2) \text{ halts }\\
       \uparrow & \text{otherwise.}
   \end{cases}\]
and 
     \[T_{{b}}(x_1,x_2,y_1,y_2) = \begin{cases}
       T_{T_{x_2}(x_1,x_2)}(y_1,y_2) & \text{ if } T_{x_2}(x_1,x_2) \text{ halts }\\
       \uparrow & \text{otherwise.}
   \end{cases}\]

    By the $S^n_m$-Theorem, there exist total computable functions $h_1$ and $h_2$, such that $T_a(x_1,x_2,y_1,y_2) = T_{h_1(x_1,x_2)}(y_1,y_2)$  and $T_b(x_1,x_2,y_1,y_2) = T_{h_2(x_1,x_2)}(y_1,y_2)$ hold for all $x_1,x_2,y_1,y_2$.
    Now let $e_1$ be the index of the Turing machine that computes the function 
    $$g_1(x_1,x_2)= f_1(h_1(x_1,x_2),h_2(x_1,x_2))\enspace,$$
    and $e_2$ be the index of the Turing machine that computes the function 
    $$g_2(x_1,x_2)= f_2(h_1(x_1,x_2),h_2(x_1,x_2))\enspace.$$
    Since $g_1$ is a total computable function, we have $T_{h_1(e_1,e_2)} {\equiv} T_{T_{e_1}(e_1,e_2)}$.  
    By definition, we have that $T_{e_1}(e_1,e_2)$ computes $f_1(h_1(e_1,e_2),h_2(e_1,e_2))$. 
    Thus $T_{h_1(e_1,e_2)} {\equiv} T_{f_1(h_1(e_1,e_2),h_2(e_1,e_2))}$. 
    Furthermore, by the same argument, we have that $T_{h_2(e_1,e_2)} {\equiv} T_{T_{e_2}(e_1,e_2)} {\equiv} T_{f_2(h_1(e_1,e_2),h_2(e_1,e_2))}$. 
    Thus if we choose $c_1 = h_1(e_1,e_2)$ and $c_2= h_2(e_1,e_2)$, we get 
    {$T_{c_1} \equiv T_{f_1(c_1,c_2)}$ and $T_{c_2}\equiv T_{f_2(c_1,c_2)}$}, as required.
\end{proof}

{We are now ready to prove Lemma~\ref{lem:twohalt_is_undecidable}, which completes the result of this section.}
\begin{proof}[of Lemma~\ref{lem:twohalt_is_undecidable}]
By way of contradiction, let's assume that there existed a program/Turing Machine $T^{\twohalt}$ that solved the $\twohalt$ problem. Recall that we had fixed an enumeration $(T_i)_{i\in\N}$ of all Turing Machines (or programs), and that we use the notaion $T(i)\downarrow$ to indicate that a Turing Machine $T$ halts on input $i$ and $T(i)\uparrow$ to indicate that $T$ loops on input $i$.
 We can now define two natural numbers $l_1$ and $l_2$ as being the indices of Turing Machines $T_{l_1}$ and $T_{l_2}$ which have the following behavior:

\[T_{l_1}(i,j, z) =\begin{cases}
    1 & \text{ if } z = 0 \\
    \uparrow &\text{ if } z=i >0 \text{ and } T^{\twohalt}(i,j)=1\\
    1  &\text{ if } z=j >0\text{ and } T^{\twohalt}(i,j)=1\\
    \uparrow &\text{ if } z=j>0 \text{ and } T^{\twohalt}(i,j)=2\\
    2  &\text{ if } z=i >0\text{ and } T^{\twohalt}(i,j)=2\\
   \uparrow & \text{ otherwise }
\end{cases}\]

\[T_{l_2}(i,j, z) =\begin{cases}
    2 &\text{if } z=0\\
    \uparrow &\text{ if } z=i>0 \text{ and } T^{\twohalt}(i,j)=1\\
    1  &\text{ if } z=j >0 \text{ and } T^{\twohalt}(i,j)=1\\
    \uparrow &\text{ if } z=j >0\text{ and } T^{\twohalt}(i,j)=2\\
    2  &\text{ if } z=i>0 \text{ and } T^{\twohalt}(i,j)=2\\
    \uparrow & \text{ otherwise }
\end{cases}\]

Now by $S^n_m$-Theorem, there are total computable functions $f_1$ and $f_2$ such that
\[T_{f_1(i,j)}(z) =\begin{cases}
    1 &\text{if } z=0\\
    \uparrow &\text{ if } z=i>0 \text{ and } T^{\twohalt}(i,j)=1\\
    1  &\text{ if } z=j>0 \text{ and } T^{\twohalt}(i,j)=1\\
    \uparrow &\text{ if } z=j>0 \text{ and } T^{\twohalt}(i,j)=2\\
    2  &\text{ if } z=i>0 \text{ and } T^{\twohalt}(i,j)=2\\
    \uparrow & \text{ otherwise }
\end{cases}\]
and
\[T_{f_2(i,j)}(z) =\begin{cases}
   2 &\text{if } z=0\\
    \uparrow &\text{ if } z=i \text{ and } T^{\twohalt}(i,j)=1\\
    1  &\text{ if } z=j \text{ and } T^{\twohalt}(i,j)=1\\
   \uparrow &\text{ if } z=j \text{ and } T^{\twohalt}(i,j)=2\\
   2  &\text{ if } z=i \text{ and } T^{\twohalt}(i,j)=2\\
    \uparrow & \text{ otherwise }
\end{cases}\]

Now, by Lemma \ref{lem:toscas_two_fold_recursion}, there are $c_1,c_2 > 0$, such that $T_{f_1(c_1,c_2)} \equiv T_{c_1}$ and $T_{f_2(c_1,c_2)}  \equiv T_{c_2}$. Recall that we use $T \equiv S$ to indicate that two Turing Machines have the same behavior as functions.

Now let us look at $T^{\twohalt}({c_1},{c_2})$. There are three options:
\begin{itemize}
    \item $T^{\twohalt}({c_1},{c_2}) = 1$. Then by definition $T_{c_1}(c_1) = T_{f_1(c_1,c_2)}(c_1) \uparrow$ and $T_{c_2}(c_2) = T_{f_2(c_1,c_2)}(c_2) \downarrow$. Thus $T^{\twohalt}$ would output $1$, the wrong answer in a case where $2$, the correct answer, was required.
    \item $T^{\twohalt}({c_1},{c_2}) = 2$. Then by definition $T_{c_1}(c_1) = T_{f_1(c_1,c_2)}(c_1) \downarrow$ and $T_{c_2}(c_2) = T_{f_2(c_1,c_2)}(c_2) \uparrow$. Thus $T^{\twohalt}$ would output $2$, the wrong answer in a case where $1$, the correct answer, was required.
    \item $T^{\twohalt}({c_1},{c_2}) \uparrow$. Note that, by definition, the Turing Machines $T_{f_1(c_1,c_2)}$ and $T_{f_2(c_1,c_2)}$ only halt and produce an output if either their input was $0$ or the call to $T^{\twohalt}$ inside their definition halts and produces an output. Thus, for this case,  we get $T_{c_1}(c_1) =  T_{f_1(c_1,c_2)}(c_1) \uparrow$ and $T_{c_2}(c_2) = T_{f_2(c_1,c_2)}(c_2) \uparrow$.  However, if both $T_{c_1}(c_1)$ and $T_{c_2}(c_2)$ don't halt, $T^{\twohalt}$ was supposed to halt. Thus $T^{\twohalt}$ also has the wrong behaviour in this case. 
\end{itemize}
In summary, the assumption that $T^{\twohalt}$ existed lead to a contradiction in all cases and thus the problem $\twohalt$ does not admit a solution. 
\end{proof}

\section{Proofs from Section~\ref{sec:oracle}}
\label{app:proofs-oracle}

\subsection{Proof of Theorem~\ref{lemma:sep-cpac-rob-loss-uncomputable}}
\label{app:proof-sep-cpac}

\begin{proof}
    Let the instance space be $\mathcal{X}= (\naturals\times (\{0\}\cup\naturals)$. 
    Let $(\varphi_i)_{i\in \naturals}$ be an enumeration of formulas and $(\pi_j)_{j\in \naturals}$ be an enumeration of proofs. 
    Let $(i,0)$, represent the $i$-th formula and for any $i, j$ let $(i,j)$ represent the $j$-th proof (independently of $i$). 
    Let $\mathcal{H}=\{h_{a,b}\}_{a,b\in \naturals \cup \{0\}}$, where
    \[h_{a,b}(i,j) = 
        \begin{cases}
            1 &\text{ if } i < a \text{ or }(i = a \text{ and } j\leq b)\\
            0 &\text{ otherwise }
        \end{cases}
        \enspace.
    \]
    It is clear that $\VC(\H)=1$ and that ERM is computably implementable, so $\H$ is properly CPAC learnable.
    
    We define the perturbation region as follows:
    \begin{align*}
        &\mathcal{U}(i,0) =\{(i,0)\}\cup \{(i,j): \pi_j \text{ proves } \varphi_i \}\enspace,\\
        &\mathcal{U}(i,j) = \{(i,0), (i,j)\} \enspace,
    \end{align*}
    which is clearly DR, as to check whether $(i,j)\in\U((i,0))$ we can check whether $\pi_j$ proves $\varphi_i$.

    It is immediate to see that $\ell^\U$ is not computable for all hypotheses in $\H$.
    Indeed, for every $i\in\N$, deciding whether for the formula $\varphi_i$ is a tautology reduces to deciding whether
    \[\ell^{\mathcal{U}}(h_{i,0}, (i,0), 1) = 0\enspace.\]

    Now, it remains to show that $\H$ is properly $\U$-robustly CPAC learnable. 
    To this end, we will first show that $\VC(\H_\mar^\U)$ is finite, and second, show that robust ERM is computable in the agnostic setting.
    By Fact~\ref{fact:suff-rob-cpac}, we will have robust CPAC learnability.
    
    To show $\VC(\H_\mar^\U)<\infty$, we consider the sets $\mar^\U_{h_{a,b}}$, which we will denote by $\mar_{a,b}$ for brevity.
    First note that either $a=\infty$ or $b=\infty$ implies that $\mar_{a,b}=\emptyset$.
    Let $a,b\in\N$ and observe that only instances of the form $(a,j)$ can belong to $\mar_{a,b}$, and thus a shattered set cannot contain both $(a,j)$ and $(a',j')$ where $a\neq a'$.
    We distinguish two cases:
    \begin{itemize}
        \item $\varphi_a$ is not a tautology: then $\mar_{a,b}=\{(a,0)\}$,
        \item $\varphi_a$ is a tautology: then $\mar_{a,b}=\{(a,0)\}\cup\{(a,j)\given j>b \wedge \pi_j \text{ proves } \varphi_a\}$.
    \end{itemize}
    From this, it is clear that $\VC(\H_\mar^\U)=1$.

    Now, to show that robust ERM is computable, 
    let $S\subseteq\X\times\{0,1\}$ be an arbitrary labelled sample of size $m$.
    For a fixed $i\in\N$ and $y\in\{0,1\}$, consider the subset $S_i^y$ of $S$, consisting of ``$i$-instances'' with label $y$, i.e., $S_i^y:=\set{((i,j),y)\given ((i,j),y)\in S}$, and let $S_i=S_i^0\cup S_i^1$.
    We can see that $h_{i,\infty}$ incurs robust loss 1 on all the instances in $S_i^0$ and robust loss 0 on all the instances in $S_i^1$.
    Furthermore, while every $h_{i, b}$ incurs robust loss 1 on all $S_i^0$, it might also incur loss 1 on some instances $S_i^1$.
    Finally, for $b\in\N$, on every instance in $S\setminus S_i$, the hypotheses $h_{i,b}$ and $h_{i,\infty}$ incur the same robust loss. 
    Thus $h_{i,\infty}$ always ``dominates'' $h_{i,b}$ in the sense that $\R_\U(h_{i,\infty};S)\leq \R_\U(h_{i,b};S)$. 
    We furthermore note that for every $i$, the loss of the function $h_{i,\infty}$ can be evaluated on every point. 
    Thus, performing RERM boils down to comparing all candidate hypotheses of the form $h_{i,\infty}$, which can be done computably.
\end{proof}

\begin{remark}
    In the above proof, {let $\H'\subset \H$ be the set of all hypotheses in $\H$ with $b=\infty$, i.e., $\H'=\set{h_{a,b}\in\H\given b=\infty}$.
    Then it is clear that, in the robust case, because of the definition of $\U$, we can get away with only outputting hypotheses from $\H'$, no matter what the underlying distribution is, and that for all hypotheses in $\H'$ the robust loss is computable.}
    In general, {for some hypothesis class $\H$, let $\H'$ be a class, such that (i) for every $h\in \H$, there is $h'\in \H'$, such that for all $(x,y)\in \X\times\Y$: $\ell^{\U}(h',x,y)\leq \ell^{\U}(h,x,y) $, (ii) the robust loss of every $h'\in \H'$ is computably evaluable. Then if $\H'$ is realizable/agnostic robustly CPAC learnable, so is $\H$. Furthermore, if $\H'\subset \H$ and $\H'$ is \emph{proper} realizable/agnostic robustly learnable, so is $\H$. }
\end{remark}

\subsection{Proof of Theorem~\ref{thm:rob-loss-comp-but-no-proper-cpac-agnostic}}
\label{app:proof-rob-loss-comp-not-rcpac}

\begin{proof}
Let the instance space be $\X= \naturals \times \naturals$.
Let the hypothesis class be $\H = \{h_{i,j}:i,j \in \naturals \}$, where 
\[h_{i,j}(x) = \begin{cases}
    1 & \text{ if } x=(i,k) \text{ with } k \leq j \\
    0 & \text{ otherwise }
\end{cases}\enspace,\]
i.e., $\H$ is the class which defines initial segments on the sets $\{(i,j)\}_{j\in\N}$ for all $i\in\N$.

We define the perturbation types as follows.
\begin{align*}
    \U((i,0)) &= \{(i,0)\} ~\cup~  \{(i,k): T_i \text{ does not halt after } k \text{ steps on the empty input }\}\enspace,\\
    \U((i,j)) &= 
    \begin{cases} 
        \{(i,j), (i,0) \} & T_i \text{ does not halt after } j  \text{ steps on the empty input} \\ \{(i,j)\} & \text{ otherwise. } 
    \end{cases}
    \enspace,
\end{align*}

We now prove properties (i)-(iii) from the theorem statement.

To prove (i), i.e., that  $\H$ is properly CPAC learnable, it suffices to observe that $\VC(\H)=1$ and that ERM is easily implementable (in both the agnostic and the realizable cases).
In order to verify $\VC(\H)=1$, let us take to arbitrary distinct points $(i_1,j_1)$ and $(i_2,j_2)$. If $i_1 \neq i_2$, then the labelling $((i_1,j_1),1),((i_2,j_2),1)$ cannot be achieved by $\H$. If $i_1=i_2$, then $j_1 \neq j_2$. Let us assume without loss of generality that $j_1 < j_2$. Then the labelling $((i_1,j_1),0),((i_2,j_2),1)$ cannot be achieved. For computably realizing ERM, we note that for any sample $S$ it is sufficient to do search over the finite class $\H_S=\{h_{i,j}: ((i,j),1)\in S\} \cup \{ h_{0}^S\}$, where $h_0^S = h_{i,0}$, with $i$ being the smallest index for which there is no $((i,j),1)\in S$. Searching over these finitely many candidate hypotheses can be done computably. 

To prove (ii), i.e., that $\H$ is properly $\U$-robustly learnable, we show that the VC dimension $\VC(\H_\mar^\U)$ of the margin class is finite.
To this end, we identify the sets $\mar^\U_\H$ and distinguish two cases:
\begin{itemize}
    \item $T_i$ does not halt after $j$ steps on the empty input: then $$\mar_{h_{i,j}}^{\U} = \{(i,0)\}\cup \{ (i,k): k > j \text{ and } T_i \text{ does not halt after }k \text{ steps}\} \enspace,$$ 
    \item $T_i$ halts on the empty input after $j$ steps:  then $\mar_{h_{i,j}}^{\U}= \emptyset $.
\end{itemize}
We now argue that $\VC(\H_\mar^\U)=1$.
To this end, let $x_1=(i_1,j_1)$ and $x_2=(i_2,j_2)$ be two distinct domain points, and note that, to be shattered, we need $i_1=i_2$, as otherwise there is no $\mar_h^\U\in \H_{\mar}^{\U}$ with $x_1, x_2\in \mar_h^\U$.
Now, without loss of generality, let $j_1<j_2$.
There are three cases:
\begin{itemize}
        \item If  $T_{i_1}$ halts after $j_2$ steps or less, then there is no $\mar_h^\U\in \H_{\mar}^{\U}$ with $x_2\in \mar_h^\U$.
        \item If $T_{i_1}$ does not halt after $j_2$ and $j_1\neq 0$, then there is no $\mar_h^\U\in \H_{\mar}^{\U}$ with $x_1\in \mar_h^\U$ and $x_2 \notin \mar_h^\U$. 
        \item If $T_{i_1}$ does not halt after $j_2$ steps and $j_1= 0$, then there is no $\mar_h^\U\in \H_{\mar}^{\U}$ with $x_2\in \mar_h^\U$ and $x_1 \notin \mar_h^\U$.
\end{itemize}
To prove (iii), i.e., that the pointwise robust loss is computably evaluable, let $h_{i,j}\in \H$, $x\in \X$, $y \in \mathcal{Y}$ be arbitrary, and consider the following procedure:
\begin{itemize}
    \item If $x=(i,0)$ and $y=1$, we have that $\ell^{\U}(h_{i,j},x,y) = 1$ if and only if  $(i,j+1) \in \U((i,0))$, which can be determined by running $T_i$ for $j+1$ steps and checking whether it halts. 
    \item If $x=(i,0)$ and $y=0$, then $\ell^{\U}(h_{i,j},x,y) = 1$.
    \item If $x= (i,k)$ with $k>0$ and $y=1$, check whether $k \leq j$. 
    If so, $\ell^{\U}(h_{i,j},x,y) = 0$, otherwise $\ell^{\U}(h_{i,j},x,y) = 1$.
    \item If $x= (i,k)$ with $k>0$ and $y=0$, check whether $k \leq j$. 
    If so, $\ell^{\U}(h_{i,j},x,y) = 1$. Otherwise $\ell^{\U}(h_{i,j},x,y) = 0$ if and only if $(i,0)\notin\U((i,j))$, which can be checked by running $T_i$ for $k$ many steps and checking whether it halts.
    \item If $x= (i',k)$ with $i'\neq i$, then $\ell^{\U}(h_{i,j},x,y) = y$.
\end{itemize}
Finally, we show that there is no strong realizable $\rerm_\H^\U$-oracle and that $\H$ is not properly agnostically $\U$-robustly CPAC learnable. We first show the non-existence of a computable strong realizable $\rerm_\H^\U$-oracle or agnostic $\rerm_\H^\U$-oracle 
We prove this with a reduction from the Halting problem: for all $i\in\N$, a strongly realizable $\rerm_\H^\U$ oracle (or agnostic $\rerm_\H^\U$-oracle) ran on the single labelled instance $((i,0),1)$ outputs a hypothesis $h$ with robust loss $\ell^{\U}(h,(i,0),1)=0$ if and only if the Turing Machine $T_i$ halts on the empty input. As a side note, we observe that a weak realizable $\rerm_\H^\U$-oracle does exist for this construction.

We show that $\H$ is not agnostically properly $\U$-robustly CPAC learnable in a similar way. Assume there was an agnostically properly $\U$-robustly CPAC learner $\A$ with sample complexity function $m$.  For all $i\in\N$, define $D_i$ as $D_i((i,0),1)=1$. Let $m=m(1/8,1/8)$. By the learning guarantee of $\A$, we have that $\A(S)=h \in \H$ with $ \mathcal{L}_{D_i}(h) \leq \underset{h\in\H}{\inf}\R(h;D_i) ~+~  \frac{1}{8} $ with probablility $7/8$ over $S\sim D_i^m$. Observe that $\underset{h\in\H}{\inf}\R(h;D_i) = 0$ if and only if $T_i$ halts on the empty input, and otherwise $\underset{h\in\H}{\inf}\R(h;D_i) = 1$. We note that $S'= (((i,0),1), \dots, ((i,0),1) )$, where $|S'|= m$ is the only possible sample drawn from $D_i$. Thus we have $\A(S')= h' \in \H$ with $ \mathcal{L}_{D_i}(h') \leq 1/8$ if and only if $T_i$ halts on the empty input. Furthermore, $\mathcal{L}_{D_i}(h') = \ell^{\U}(h',(i,0),1) \in \{0,1\} $ for all $h'\in\H$. Thus we have  $\mathcal{L}_{D_i}(\A(S'),(i,0), 1) = \ell^{\U}(\A(S'),(i,0), 1) = 0$ if and only if  $T_i$ halts on the empty input. As established above, the robust loss $\ell^{\U}$ can be computably evaluated on any $h\in\H$. Thus a proper agnostic learner $\A$ yields a computable procedure for solving the Halting problem, which is a contradiction to the computable undedicability of the Halting problem. Thus $\H$ is not agnostically properly $\U$-robustly CPAC learnable.
\end{proof}

\subsection{Proof of Proposition~\ref{prop:proper->rerm}}
\label{app:proof-prop:proper->rerm}

\begin{proof}
    Let $S=((x_1,y_1),\dots, (x_k,y_k))$ be an input sample for $\rerm^\U_\H$, and let $\A$ be the robust SCPAC learning algorithm with computable sample complexity function $m_{\H,\U}$.
    Consider the uniform distribution $D_S$ over $S$. 
    Let $m\geq m_{\H,\U}(\frac{1}{k+1}, 1/7)$, {which is computable by the SCPAC guarantee}.
    There are exactly $k^m$ different possible samples $S_1,\dots S_{k^m}$ of size $m$ {that can be drawn from} $D_S$, all of which being equally likely. 
    Using the proper learner $\A$ on all samples $S_i$, we can generate the class $\hat{\H}= \{\A(S_i): i \in [k^m]\}$. 
    By the robust learning guarantee, we know that at least 7/8 of the hypotheses in $\hat{\H}$ are successful hypotheses, i.e., hypotheses within robust risk $1/(k+1)$ of the optimal one. 
    Furthermore, since $\epsilon = \frac{1}{k+1}$, any successful hypothesis must have optimal loss on $S$ (as any error on a sample point would incur loss $\frac{1}{k}$).
    It now suffices to use $\ell^\U$ on $\hat{\H}$ and return any one of the hypotheses with minimal empirical robust risk on $S$.
    {Note that, for the robust realizable setting (the second case in the theorem statement), we only get a weak realizable oracle $\rerm_\H^\U$, as if the sample is not robustly realizable, we don't have any guarantees on $\A$'s behaviour.}
\end{proof}

\section{Proofs from Section~\ref{sec:c-dim-u}}

\subsection{Proof of Lemma~\ref{lemma:rob-nflt}}
\label{app:nflt}

\begin{proof}
The idea is similar to the standard No-Free-Lunch Theorem. 
We let $m$ and $\zdef[2m]$ as in the theorem statement, meaning that for all $i$, $\U(\zin)\cap\U(\zip)\neq \emptyset$ and for all $i\neq j$ the regions $\U(\zip)$ and $\U(\zjn)$ are disjoint. 
We now look at a robust labelling $y\in\{0,1\}^{2m}$.
There are $T=2^{2m}$ such labellings, and for each $j\in[T]$ associated with labelling $y^{(j)}\in\{0,1\}^{2m}$, there exists a function $f_j$ such that for all $i\in[2m]$, for all $z'\in\U(z_i^{y_i^{(j)}})$ we have that $f_j(z')=y_i^{(j)}$.
Namely, for $i\in[2m]$, the label $y_i^{(j)}$ dictates which of $\zin$ or $\zip$ will have its perturbation region constantly (and thus robustly) labelled.
For each function $f_j$ and its robust labelling $y^{(j)}$, we define the distribution $D_j$ as follows:
\begin{equation*}
    D_j((z,y))=
    \begin{cases}
        1/2m    &\text{if $y=y_i^{(j)}$ and $z=z_i^y$ for some $i\in[2m]$}\\
        0       &\text{otherwise}
    \end{cases}
    \enspace.
\end{equation*}
It is easy to check that $\R_\U(f_j;D_j)=0$. 

Let $X=\{x_i\}_{i=1}^m$ be some $\U$-robustly shatterable set corresponding  to the set of pairs $\zdef[2m]$, that is $x_i\in \U(\zin)\cap\U(\zip)$ for each $i\in[2m]$. Now note any hypothesis $h:\X\to\{0,1\}$ (that the learner may output) has to assign labels to the the points in the shatterable set $X$ and $h(x_i) = 1$ implies $\ell^{\U}(h, \zin) = 1$ while $h(x_i) = 0$ implies $\ell^{\U}(h, \zip) = 1$. Thus each hypothesis incurs robust loss on at least one of the members $\{\zin, \zip\}$ of each pair, and such a member can be deduced from observing $h(x_i)$. Now the remaining argument to show that there exist $f_j,D_j$ such that the conditions of the theorem statement are met is exactly as in the proof of the standard No-Free-Lunch theorem (see, e.g., Theorem~5.1 in \cite{shalev2014understanding} for details). 
We note that the argument there implies that the high expected loss of the learner's output is witnessed by its behaviours on the shatterable set $X$. Namely, there exists $f_j, D_j$ such that
  \begin{align*}
        \eval{S\sim D_j^m}{\R_\U(\A(S);D_j)}
        =\frac{1}{k}\sum_{l=1}^k \R_\U(\A(S^j_l);D_j)
        & =\frac{1}{k}\sum_{l=1}^k \frac{1}{2m}\sum_{i=1}^{2m}\ell^\U\left(\A(S^j_l),z_i^{(j)},y_i^{(j)}\right)
        \enspace,\\
        & \geq \frac{1}{k}\sum_{l=1}^k \frac{1}{2m}\sum_{i=1}^{2m}\indct{\A(S^j_l)(x_i) \neq y_i^{(j)}} \geq 1/8
    \end{align*}
where $k=(2m)^{m}$ is the number of possible sequences of length $m$  drawn from $\supp(D_j):=\set{(z_i^{(j)},y_i^{(j)})}_{i=1}^{2m}$, and where we have denoted by $S^j_l$ the $l$-th such sequence.
\end{proof}

\subsection{Proof of Theorem~\ref{thm:nec-rob-cpac}}
\label{app:proof-nec-rob-cpac}
\begin{proof}
    First note that, if the perturbation type $\U$ is such that there is an upper bound $M$ on the maximal number of pairs $\zdef[2m]$ admitting a robustly shatterable set, then the computable robust shattering dimension is vacuously finite for any hypothesis class $\H$. 
    Otherwise, let $\A$ be a $\U$-robust CPAC learner for $\H$ with sample complexity function $m(\epsilon,\delta)$.
    Then, fixing $\epsilon=1/8$ and $\delta=1/7$, we have that for all distributions on $\X\times\{0,1\}$, for any sample of size $m:=m(1/8,1/7)$
    \begin{equation}
    \label{eqn:rob-cpac-guarantee}
        \prob{S\sim D^m}{\R_\U(\A(S);D)\geq\underset{h\in\H}{\min}\; \R_\U(h;D)+1/8}< 1/7\enspace.
    \end{equation}
    
    Now, let $\zdef[2m]$ admit a robustly shatterable set $X$, and note that, as explained in the proof of Lemma~\ref{lemma:c-rob-nflt}, we can computably find such a set $X$. 
    Note that $\H$ and $\U$ satisfying the conditions of the theorem statement imply, by Lemma~\ref{lemma:c-rob-nflt}, that we can find a distribution $D$ on $Z\times\{0,1\}$ such that 
    \begin{equation}
    \label{eqn:c-rob-nflt-guarantee}
        \prob{S\sim D^m}{\R_\U(\A(S);D)\geq 1/8}\geq 1/7\enspace,
    \end{equation}
    and whose support is on some $\set{z_i^{y_i}}_{i=1}^{2m}$ for $y\in\{0,1\}^{2m}$.
    This support induces a labelling $y$ on $X$ that is not robustly achievable with respect to $Z$ by any $h\in\H$, as we would otherwise have $\underset{h\in\H}{\min}\;L_D^\U(h)=0$, implying that Equation~\ref{eqn:c-rob-nflt-guarantee} and $\prob{S\sim D^m}{\R_\U(\A(S);D)\geq 1/8}< 1/7$ must hold simultaneously, a contradiction.
\end{proof}

\section{Remarks on the robust shatterability in the impossibility results from Section~\ref{sec:relating-cpac-rob-cpac}}
\label{app:comment-impossibility}
We now comment on the robust shatterability in the impossibility results from Section~\ref{sec:relating-cpac-rob-cpac}.
On the one hand, the perturbation type in Theorem~\ref{thm:proper-rob-cpac-imposs-realizable} does not allow for sets of arbitrary size to admit a robustly shatterable set.
Indeed, any sequence of halting Turing machines (which are associated with the only instances in $\N$ that have a strict inequality $k\subset\U(k)$) will have an ordering such that their perturbation regions form an ascending chain, which makes the condition $i\neq j \implies \U(\zip)\cap\U(\zjn)=\emptyset$ unsatisfiable for any candidate set $\zdef[2m]$. 
On the other hand, the perturbation type in Theorem~\ref{thm:rob-cpac-imposs-agnostic} satisfies the robust shatterability requirement. 
Indeed, we will add the mild assumption that each proof proves at most one formula.
This implies that for each $i\neq i'$, $\U(6i+4)$ and $\U(6i'+4)$ are disjoint, thus, letting $T=(\varphi_{i_k})_{k\in \N}$ be an enumeration of the tautologies, the set $Z=\{(6i_k+2,6i_k+4)\}_{k\in\N}$ admits the robustly shatterable set $X=\{6i_k+2\}_{k\in\N}$, while the sets $\U(6i+4)$ may be infinite.

\end{document}